\documentclass{article} % For LaTeX2e

\usepackage{microtype}
\usepackage{graphicx}
\usepackage{subfigure}
\usepackage{booktabs} % for professional tables

\usepackage{amsthm}
\usepackage{mathtools}

\newtheorem{theorem}{Theorem}[section]
\newtheorem{lemma}[theorem]{Lemma}
\newtheorem{corollary}[theorem]{Corollary}
\newtheorem{proposition}[theorem]{Proposition}
\newtheorem{definition}[theorem]{Definition}

\newcommand{\sR}{\mathbb{R}}

\usepackage{amsmath,amsfonts,bm}

\usepackage{hyperref}

\usepackage[accepted]{icml2019}

\icmltitlerunning{A Quantitative Analysis of the Effect of Batch Normalization on Gradient Descent}

\begin{document}

\twocolumn[
% \icmltitle{On the Convergence and Robustness of Batch Normalization}
\icmltitle{
A Quantitative Analysis of the Effect of Batch Normalization on Gradient Descent
}




\begin{icmlauthorlist}
\icmlauthor{Yongqiang Cai}{nus}
\icmlauthor{Qianxiao Li}{nus,astar}
\icmlauthor{Zuowei Shen}{nus}
\end{icmlauthorlist}

\icmlaffiliation{nus}{Department of Mathematics, National University of Singapore, Singapore}
\icmlaffiliation{astar}{Institute of High Performance Computing, A*STAR, Singapore}

\icmlcorrespondingauthor{Yongqiang Cai}{matcyon@nus.edu.sg}
\icmlcorrespondingauthor{Qianxiao Li}{qianxiao@nus.edu.sg}
\icmlcorrespondingauthor{Zuowei Shen}{matzuows@nus.edu.sg}

%\icmlkeywords{Machine Learning, ICML}
\icmlkeywords{Batch Normalization, Ordinary Least Square, Gradient Descent, Linear Convergence}

\vskip 0.3in
]



\printAffiliationsAndNotice{}  % leave blank if no need to mention equal contribution

\begin{abstract}
    Despite its empirical success and recent theoretical progress, there generally lacks a quantitative analysis of the effect of batch normalization (BN) on the convergence and stability of gradient descent. In this paper, we provide such an analysis on the simple problem of ordinary least squares (OLS). Since precise dynamical properties of gradient descent (GD) is completely known for the OLS problem, it allows us to isolate and compare the additional effects of BN. More precisely, we show that unlike GD, gradient descent with BN (BNGD) converges for arbitrary learning rates for the weights, and the convergence remains linear under mild conditions. Moreover, we quantify two different sources of acceleration of BNGD over GD -- one due to over-parameterization which improves the effective condition number and another due having a large range of learning rates giving rise to fast descent. These phenomena set BNGD apart from GD and could account for much of its robustness properties. These findings are confirmed quantitatively by numerical experiments, which further show that many of the uncovered properties of BNGD in OLS are also observed qualitatively in more complex supervised learning problems.

\end{abstract}

\section{Introduction}

    Batch normalization (BN) is one of the most important techniques for training deep neural networks and has proven extremely effective in avoiding gradient blowups during back-propagation and speeding up convergence. In its original introduction~\citep{Ioffe2015Batch}, the desirable effects of BN are attributed to the so-called ``reduction of covariate shift''. However, it is unclear what this statement means in precise mathematical terms.

    Although recent theoretical work have established certain convergence properties of gradient descent with BN (BNGD) and its variants~\citep{Ma2017Convergence, Kohler2018Towards, arora2018theoretical}, there generally lacks a quantitative comparison between the dynamics of the usual gradient descent (GD) and BNGD. In other words, a basic question that one could pose is: what quantitative changes does BN bring to the stability and convergence of gradient descent dynamics? Or even more simply: why should one use BNGD instead of GD? To date, a general mathematical answer to these questions remain elusive. This can be partly attributed to the complexity of the optimization objectives that one typically applies BN to, such as those encountered in deep learning. In these cases, even a quantitative analysis of the dynamics of GD itself is difficult, not to mention a precise comparison between the two.

    For this reason, it is desirable to formulate the simplest non-trivial setting, on which one can concretely study the effect of batch normalization and answer the questions above in a quantitative manner.
    This is the goal of the current paper, where we focus on perhaps the simplest supervised learning problem -- ordinary least squares (OLS) regression -- and analyze precisely the effect of BNGD when applied to this problem. A primary reason for this choice is that the dynamics of GD in least-squares regression is completely understood, thus allowing us to isolate and contrast the additional effects of batch normalization.

    Our main findings can be summarized as follows
    \begin{enumerate}
        \item Unlike GD, BNGD converges for arbitrarily large learning rates for the weights, and the convergence remains linear under mild conditions.
        \item The asymptotic linear convergence of BNGD is faster than that of GD, and this can be attributed to the over-parameterization that BNGD introduces.
        % \item  The asymptotic linear convergence rate of BNGD is higher than that of GD, and this can be attributed to the over-parameterization that BNGD introduces.
        \item Unlike GD, the convergence rate of BNGD is insensitive to the choice of learning rates. The range of insensitivity can be characterized, and in particular it increases with the dimensionality of the problem.
    \end{enumerate}
    Although these findings are established concretely only for the OLS problem, we will show through numerical experiments that some of them hold qualitatively, and sometimes even quantitatively for more general situations in deep learning.

    \subsection{Related Work}

    Batch normalization was originally introduced in~\citet{Ioffe2015Batch} and subsequently studied in further detail in~\citet{Ioffe2017Batch}. Since its introduction, it has become an important practical tool to improve stability and efficiency of training deep neural networks~\citep{Bottou2018Optimization}.
    %In~\citep{Cooijmans2016Recurrent}, authors propose a reparameterization of LSTM that brings the benefits of batch normalization to recurrent neural networks.
    Initial heuristic arguments attribute the desirable features of BN to concepts such as ``covariate shift'', but alternative explanations based on landscapes~\citep{Santurkar2018How} and effective regularization~\citep{Bjorck2018Understanding} have been proposed.

    Recent theoretical studies of BN include~\citet{Ma2017Convergence,Kohler2018Towards,arora2018theoretical}. We now outline the main differences between them and the current work.
    In~\citet{Ma2017Convergence}, the authors proposed a variant of BN, the diminishing batch normalization (DBN) algorithm and established its convergence to a stationary point of the loss function. In~\citet{Kohler2018Towards}, the authors also considered a BNGD variant by dynamically setting the learning rates and using bisection to optimize the rescaling variables introduced by BN. It is shown that this variant of BNGD converges linearly for simplified models, including an OLS model and ``learning halfspaces''.
    The primary difference in the current work is that we do not dynamically modify the learning rates, and consider instead a constant learning rate, i.e. the original BNGD algorithm. This is an important distinction; While a decaying or dynamic learning rate is sometimes used in GD, in the case of BN it is critical to analyze the constant learning rate case, precisely because one of the key practical advantages of BN is that a big learning rate can be used. Moreover, this allows us to isolate the influence of batch normalization itself, without the potentially obfuscating effects a dynamic learning rate schedule can introduce (e.g.\,see Eq.~\eqref{eq:eps_hat} and the discussion that follows).
    As the goal of considering a simplified model is to analyze the additional effects purely due to BN on GD, it is desirable to perform our analysis in this regime.

    In~\citet{arora2018theoretical}, the authors proved a general convergence result for BNGD of $\mathcal{O}(k^{-1/2})$ in terms of the gradient norm for objectives with Lipschitz continuous gradients. This matches the best result for gradient descent on general non-convex functions with learning rate tuning \citep{Carmon2017Lower}. In contrast, our convergence result is in iteration and is shown to be linear under mild conditions (Theorem \ref{th:convergence_rate_BNGD}). This convergence result is stronger, but this is to be expected since we are considering a specific case. More importantly, we discuss concretely how BNGD offers advantages over GD instead of just matching its best-case performance. For example, not only do we show that convergence occurs for any learning rate, we also derive a quantitative relationship between the learning rate and the convergence rate, from which the robustness of BNGD on OLS can be explained (see Section ~\ref{sec:analysis}).

    \subsection{Organization}

    Our paper is organized as follows.
    In Section~\ref{sec:background}, we outline the ordinary least squares (OLS) problem and present GD and BNGD as alternative means to solve this problem. In Section~\ref{sec:analysis}, we demonstrate and analyze the convergence of the BNGD for the OLS model, and in particular contrast the results with the behavior of GD, which is completely known for this model. We also discuss the important insights to BNGD that these results provide us with. We then validate these findings on more general supervised learning problems in Section~\ref{sec:experiments}. Finally, we conclude in Section~\ref{sec:conclusion}.

\section{Background} %Preliminaries
    \label{sec:background}

    \subsection{Ordinary Least Squares and Gradient Descent}

    Consider the simple linear regression model where $x \in \sR^{d}$ is a random input column vector and $y$ is the corresponding output variable. Since batch normalization is applied for each feature separately, in order to gain key insights it is sufficient to consider the case $y \in \sR$. A noisy linear relationship is assumed between the dependent variable $y$ and the independent variables $x$, i.e.\,$y = x^T w + \text{noise}$ where $w \in \sR^{d}$ is the vector of trainable parameters.
    Denote the following moments:
    \begin{align}%\label{eq:}
        H := E[xx^T], \quad g := E[xy], \quad c:= E[y^2].
    \end{align}
    To simplify the analysis, we assume the covariance matrix $H$ of $x$ is positive definite and the mean $E[x]$ of $x$ is zero. The eigenvalues of $H$ are denoted as $\lambda_i(H), i=1,2,...d,$. Particularly, the maximum and minimum eigenvalue of $H$ is denoted by $\lambda_{max}$ and $\lambda_{min}$ respectively. The condition number of $H$ is defined as $\kappa := \tfrac{\lambda_{max}}{\lambda_{min}}$.
    Note that the positive definiteness of $H$ allows us to define the vector norm $\|.\|_H$ by $\|x\|^2_H = x^THx$.

    The ordinary least squares (OLS) method for estimating the unknown parameters $w$ leads to the following optimization problem, %objective function
    \begin{align}\label{eq:OLS}
        \min_{w \in \sR^d} J_0(w) :&= \tfrac12 E_{x,y}[ (y-x^Tw)^2 ]\\
        &=\tfrac{c}{2} - w^T g + \tfrac12 w^T H w,\nonumber
    \end{align}
    %The gradient of $J_0$ with respect to $w$ is $\nabla_w J_0(w) = Hw - g$, and
    which has unique minimizer $w=u:=H^{-1}g$.

    The gradient descent (GD) method (with step size or learning rate $\varepsilon$) for solving the optimization problem (\ref{eq:OLS}) is given by the iteration
    \begin{align}\label{eq:GD_w}
        w_{k+1} = w_k - \varepsilon \nabla_w J_0(w_k)  = (I-\varepsilon H)w_k + \varepsilon g,
    \end{align}
    which converges if $0 < \varepsilon < \tfrac2{\lambda_{max}}=:\varepsilon_{max}$, and the convergence rate is determined by the spectral radius $\rho_\varepsilon := \rho(I-\varepsilon H) = \max_i \{|1-\varepsilon \lambda_i(H)|\}$ with
    \begin{align}\label{eq:GD_residual_recurr}
        \|u - w_{k+1}\| \le \rho(I-\varepsilon H) \|u - w_k\|.
    \end{align}
    It is well-known (e.g.\,see Chapter 4 of~\citet{Saad2003Iterative}) that the optimal learning rate is $\varepsilon_{opt} = \tfrac2{\lambda_{max}+\lambda_{min}}$, where the optimal convergence rate is $\rho_{opt} = \tfrac{\kappa-1}{\kappa+1}$.
    % \begin{align}\label{eq:}
    %     \|u - w_{k+1}\| \le \tfrac{\kappa-1}{\kappa+1} \|u - w_k\|.
    % \end{align}

    \subsection{Batch Normalization}
    Batch normalization is a feature-wise normalization procedure typically applied to the output, which in this case is simply $z=x^T w$. The normalization transform is defined as follows:
    \begin{align}%\label{eq:}
        N(z) := \tfrac{z-E[z]}{\sqrt{\text{Var}[z]}}
        =\tfrac{x^T w}{\sigma},
    \end{align}
    where $\sigma := \sqrt{w^T H w}$. After this rescaling, $N(z)$ will be order 1, and hence in order to reintroduce the scale~\citep{Ioffe2015Batch}, we multiply $N(z)$ with a rescaling parameter $a$ (Note that the shift parameter can be set zero since $\mathbb{E}[w^T x | w] = 0$). Hence, we get the BN version of the OLS problem (\ref{eq:OLS}):
    \begin{align}\label{eq:OLS_BN}
        \min_{w \in \sR^d, a \in \sR} J(a,w) :&=
        \tfrac12 E_{x,y}\big[ \big( y-a N(x^Tw) \big)^2 \big] \nonumber\\
        &=\tfrac{c}{2} - \tfrac{w^Tg}{\sigma}a + \tfrac12 a^2.
    \end{align}
    The objective function $J(a,w)$ is no longer convex. In fact, it has critical points, $\{(a^*, w^*)|a^*=0, w^{*T}g=0\}$, which are saddle points of $J(a,w)$ if $g\neq0$.

    We are interested in the critical points which constitute the set of global minima and satisfy the relations
    \begin{align*}%\label{eq:}
        a^* = \mathrm{sign}(s) \sqrt{u^{T} H u},%\\
        w^* = s u, \text{ for some } s \in\sR \setminus \{ 0 \}.
    \end{align*}
    %where $u = H^{-1} g$, and the sign of $a^*$ is depend on the direction of $u, w^*$, i.e.\,$\mathrm{sign}(a^*) = \mathrm{sign}(u^{T}w^*)$.
    It is easy to check that they are in fact global minimizers and the Hessian matrix at each point is degenerate. Nevertheless, the saddle points are strict (see appendix \ref{sec:matrix}), which typically simplifies the analysis of gradient descent on non-convex objectives~\citep{Lee2016Gradient, Panageas2017Gradient}.

    We consider the gradient descent method for solving the problem (\ref{eq:OLS_BN}), which we hereafter call batch normalization gradient descent (BNGD). We set the learning rates for $a$ and $w$ to be $\varepsilon_a$ and $\varepsilon$ respectively. These may be different, for reasons which will become clear in the subsequent analysis. We thus have the following discrete-time dynamical system:
    \begin{align}\label{eq:BNGD_a}
        a_{k+1} &= a_k + \varepsilon_a \Big( \tfrac{w_k^T g}{\sigma_k} - a_k \Big),\\
        \label{eq:BNGD_w}
        w_{k+1} &= w_k  +  \varepsilon \tfrac{a_k}{\sigma_k} \Big( g- \tfrac{w_k^T g}{\sigma_k^2}  H w_k \Big).
    \end{align}
    To simplify subsequent notation, we denote by $H^*$ the matrix
    \begin{align}%\label{eq:}
        H^* := H - \tfrac{Huu^TH}{u^THu},
    \end{align}
    We will see later that the over-parameterization introduced by BN gives rise to a degenerate Hessian matrix
    $\text{diag}\big(1,\tfrac{\|u\|^2}{\|w^*\|^2} H^* \big)$ at a minimizer $(a^*,w^*)$, and the BNGD dynamics is governed by $H^*$ instead of $H$ as in the GD case. The matrix $H^*$ is positive semi-definite ($H^*u = 0$) and has better spectral properties than $H$, such as a lower effective condition number $\kappa^* = \tfrac{\lambda_{max}^*}{\lambda_{min}^*} \le \kappa$, where $\lambda_{max}^*$ and $\lambda_{min}^*$ are the maximal and minimal nonzero eigenvalues of $H^*$ respectively.
    Particularly, $\kappa^* < \kappa$ for almost all $u$ (see appendix~\ref{sec:matrix}).

\section{Mathematical Analysis of BNGD on OLS}
    \label{sec:analysis}

    In this section, we discuss several mathematical results one can derive concretely for BNGD on the OLS problem (\ref{eq:OLS_BN}).

    Compared with GD, the update coefficient before $Hw_k$ in Eq.~(\ref{eq:BNGD_w}) changed from $\varepsilon$ in Eq.~\eqref{eq:GD_w} to a complicated term which we call the \emph{effective learning rate} $\hat \varepsilon_k$
    \begin{align}\label{eq:eps_hat}
         \hat \varepsilon_k
         :=
         \varepsilon \tfrac{a_k}{\sigma_k} \tfrac{w_k^T g}{\sigma_k^2}.
    \end{align}
    Also, notice that with the over-parameterization introduced by $a$, it is no longer necessary for $w_k$ to converge to $u$. In fact, any non-zero scalar multiple of $u$ can be a global minimum. Hence, instead of considering the residual $u-w_k$ as in the GD analysis Eq.~(\ref{eq:GD_residual_recurr}), we may combine Eq.~\eqref{eq:BNGD_a} and Eq.~\eqref{eq:BNGD_w} to give
    \begin{align}\label{eq:diffuw_BN}
        u - \tfrac{w_k^T g}{\sigma_k^2}w_{k+1}
        =
        (I - \hat\varepsilon_k H )\Big( u- \tfrac{w_k^T g}{\sigma_k^2}  w_k \Big).
    \end{align}
    Define the modified residual $e_k := u - ({w_k^T g}/{\sigma_k^2})w_{k}$, which equals $0$ if and only if $w_k$ is a global minimizer. Observe that the mapping $u \mapsto ({w^T g}/{\sigma^2})w = ({w^T H u}/{w^T H w})w$ is an orthogonal projection under the inner product induced by $H$, hence we immediately have
    % Using the property of $H$-norm (see section~\ref{sec:matrix}), we find that the effective learning rate $\hat \varepsilon_k$ determines the convergence rate of $e_k$ (and hence the loss function, see Lemma~\ref{lemma:convergence_loss}) via
    % \begin{align}\label{eq:convergence_rate}
    %     \|e_{k+1}\|_H
    %     \le
    %     \rho(I - \hat\varepsilon_k H )\|e_k\|_H,
    % \end{align}
    \begin{align}\label{eq:convergence_rate}
        \|e_{k+1}\|_H
        \le
        \big\| u - \tfrac{w_k^T g}{\sigma_k^2}w_{k+1} \big\|_H
        \le
        \rho(I - \hat\varepsilon_k H )\|e_k\|_H,
    \end{align}
    where $\rho(I - \hat\varepsilon_k H )$ is spectral radius of the matrix $I - \hat\varepsilon_k H$. In other words, as long as $\max_i \{|1-\hat\varepsilon_k \lambda_i(H)|\} \leq \hat\rho < 1$ for some $\hat\rho < 1$ and all $k$, we have linear convergence of the residual (which also implies linear convergence of the objective, see appendix Lemma~\ref{lemma:convergence_loss}).

    At this point, we make an important observation: if we allow for dynamic learning rates, we may simply set $\hat\varepsilon_k = c$ for some fixed $c \in (0, 2/\lambda_{max})$ at every iteration. Then, linear convergence is immediate. However, it is clear that this fast convergence is almost entirely due to the effect of dynamic learning rates, and this has limited relevance in explaining the effect of BN.
    Moreover, comparing with Eq.~\eqref{eq:GD_residual_recurr} one can observe that with this choice, BNGD and GD have the same optimal convergence rates, and so this cannot offer explanations for any advantage of BNGD over GD either. For these reasons, it is important to avoid such dynamic learning rate assumptions.
    % The inequality (\ref{eq:convergence_rate}) shows that if we enforce $\hat\varepsilon_k = 1/\lambda_{max}$ for each $k$, which is done in the analysis in~\cite{Kohler2018Towards}, then one immediately obtains the same linear convergence rate with GD. But this requires knowledge of $\lambda_{max}$ (problem-dependent). We instead focus our analysis on the original BNGD algorithm where the step size is fixed.

    As discussed above, without using dynamic learning rates one has to then estimate $\hat\varepsilon_k$ to establish convergence. Heuristically, observe that if $\varepsilon$ small enough, this is likely true as the other terms can be controlled due to the normalization. Thus, convergence for small $\varepsilon$ should hold.
    In order to handle the large $\varepsilon$ case, we establish a simple but useful scaling law that draws connections amongst cases with different $\varepsilon$ scales.

    % In the following sections, we establish a simple but useful scaling property, which an important ingredient in allowing us to prove a linear convergence result for arbitrary constant learning rates. We also derive the asymptotic properties of the effective learning rate of BNGD, which shows some interesting sensitivity behavior of BNGD on the chosen learning rates. Detailed proofs of all results presented here can be found in the supplementary document.

    \subsection{Scaling Property}

    The dynamical properties of the BNGD iterations are governed by a set of parameters, or a \emph{configuration} $\{H, u, a_0, w_0, \varepsilon_a,\varepsilon \}$.

    \begin{definition}[Equivalent configuration] Two configurations, $\{H, u, a_0, w_0, \varepsilon_a,\varepsilon \}$ and $\{H', u', a_0', w_0', \varepsilon_a',\varepsilon' \}$, are said to be equivalent if for BNGD iterates $\{w_k\}$, $\{w'_k\}$ following these configurations respectively, there is an invertible linear transformation $T$ and a nonzero constant $t$ such that $w_k' = Tw_k, a'_k=ta_k$ for all $k$.
    \end{definition}

    The scaling property ensures that equivalent configurations must converge or diverge together, with the same rate up to a constant multiple. Now, it is easy to check the system has the following scaling law.
    \begin{proposition}[Scaling property]
    \label{prop:scaling}
    Suppose $\mu\neq0,\gamma\neq0,r\neq0,Q^TQ=I$, then
    % \begin{itemize}
    %     \item[(1)] The configurations
    %     $\{\mu Q^THQ,\tfrac{\gamma}{\sqrt{\mu}} Qu, \gamma a_0, \gamma Q w_0, \varepsilon_a, \varepsilon\}$
    %     and
    %     $\{H,u,a_0,w_0,\varepsilon_a,\varepsilon\}$
    %     are equivalent.
    %     \item[(2)] The configurations
    %     $\{H,u,a_0,w_0,\varepsilon_a,\varepsilon\}$ and
    %     $\{H,u,a_0,r w_0,\varepsilon_a,r^2 \varepsilon\}$
    %     are equivalent.
    % \end{itemize}
   (1) The configurations
        $\{\mu Q^THQ,\tfrac{\gamma}{\sqrt{\mu}} Qu, \gamma a_0, \gamma Q w_0, \varepsilon_a, \varepsilon\}$
        and
        $\{H,u,a_0,w_0,\varepsilon_a,\varepsilon\}$
        are equivalent.
(2) The configurations
        $\{H,u,a_0,w_0,\varepsilon_a,\varepsilon\}$ and
        $\{H,u,a_0,r w_0,\varepsilon_a,r^2 \varepsilon\}$
        are equivalent.
    \end{proposition}

    It is worth noting that the scaling property (2) in Proposition~\ref{prop:scaling} originates from the batch-normalization procedure and is independent of the specific structure of the loss function. Hence, it is valid for general problems where BN is used (appendix Lemma~\ref{prop:scaling_general}). Despite being a simple result, the scaling property is important in determining the dynamics of BNGD, and is useful in our subsequent analysis of its convergence and stability properties. For example, it indicates that separating learning rate for weights ($w$) and rescaling parameters ($a$) is equivalent to changing the norm of initial weights.

    \subsection{Batch Normalization Converges for Arbitrary Step Size}

    Having established the scaling law, we then have the following convergence result for BNGD on OLS.

    \begin{theorem}[Convergence of BNGD]
    \label{th:convergence_BNGD}
    The iteration sequence $(a_k,w_k)$ in Eq.~(\ref{eq:BNGD_a})-(\ref{eq:BNGD_w}) converges to a stationary point for any initial value $(a_0,w_0)$ and any $\varepsilon>0$, as long as $\varepsilon_a \in (0,1]$.
    %Particularly, we have the following sufficient conditions of converging to global minimizers.
    Particularly, we have:
  If $\varepsilon_a=1$ and $\varepsilon>0$, then $(a_k,w_k)$ converges to global minimizers for almost all initial values $(a_0,w_0)$.
    % \begin{itemize}
    % %   \item[(1)] If $a_0 w_0^Tg>0$ (or $a_0=0, w_0^Tg\neq0$), $\varepsilon_a \in (0,1]$ and $\varepsilon$ is sufficiently small (the smallness is quantified by Lemma~\ref{lemma:convergence_BNGD_small2}), then $(a_k,w_k)$ converges to a global minimizer.
    %   \item[] If $\varepsilon_a=1$ and $\varepsilon>0$, then $(a_k,w_k)$ converges to global minimizers for almost all initial values $(a_0,w_0)$.
    % %   \item[(2)] The set of initial values $(a_0, w_0)$ such that  $(a_k, w_k)$ converges to a minimizer $(a^*, w^*)$ with effective learning rate $\hat\varepsilon:=\lim\limits_{k \to \infty} \hat\varepsilon_k > \varepsilon_{max}^*$ and $\det(I-\hat\varepsilon H^*) \neq 0$ is of measure zero.
    % \end{itemize}
    \end{theorem}
    \begin{proof}[Sketch of proof]
        We first prove that the algorithm converges for any $\varepsilon_a \in (0,1]$ and small enough $\varepsilon$, with any initial value $(a_0,w_0)$ such that $\|w_0\| \ge 1$ (appendix Lemma~\ref{lemma:convergence_BNGD_small2}). Next, we observe that the sequence $\{\| w_k \|\}$ is monotone increasing, and thus either converges to a finite limit or diverges. The scaling property is then used to exclude the divergent case – if $\{ \|w_k\| \}$ diverges, then at some $k$ the norm $\|w_k\|$ should be large enough, and by the scaling property, it is equivalent to a case where $\|w_k\|=1$ and $\varepsilon$ is small, which we have proved converges. This shows that {$\|w_k\|$} converges to a finite limit, from which the convergence of $w_k$ and the loss function value can be established, after some work. This proof is fully presented in appendix Theorem~\ref{th:convergence_BNGD_general} and the preceding lemmas.
        Lastly, using the ``strict saddle point'' arguments~\citep{Lee2016Gradient, Panageas2017Gradient}, we can prove the set of initial value for which $(a_k,w_k)$ converges to saddle points has Lebesgue measure 0, provided $\varepsilon_a=1, \varepsilon>0$ (appendix Lemma~\ref{lemma:strict_saddle}).
    \end{proof}

    It is important to note that BNGD converges for all step size $\varepsilon >0$ of $w_k$, independent of the spectral properties of $H$. This is a significant advantage and is in stark contrast with GD, where the step size is limited by $2/\lambda_{\max}$, and the condition number of $H$ intimately controls the stability and convergence rate. Although we only prove the almost everywhere convergence to a global minimizer for the case of $\varepsilon_a=1$, we have not encountered convergence to saddles in the OLS experiments even for $\varepsilon_a \in (0,2)$ with initial values $(a_0,w_0)$ drawn from typical distributions.

    \textbf{Remark:}
    In appendix \ref{sec:general_obj}, we show that the combination of the scaling property and the monotonicity of weight norms, which hold for batch (and weight) normalization of general loss functions, can be used to prove a more general convergence result: if iterates converge for small enough $\varepsilon$, then gradient norm converges for any $\varepsilon$. We note that in the independent work of~\citet{arora2018theoretical}, similar ideas have been used to prove convergence results for batch normalization for neural networks. Lastly, one can also show that in the general case, the over-parameterization due to batch (and weight) normalization only introduces strict saddle points (see appendix Lemma \ref{lemma:introducing_strict_saddle}).

    \subsection{Convergence Rate and Acceleration Due to Over-parameterization}
    \label{subsec:convergence_rate}

    Having established the convergence of BNGD on OLS, a natural follow-up question is why should one use BNGD over GD. After all, even if BNGD converges for any learning rate, if the convergence is universally slower than GD then it does not offer any advantages. We prove the following result that shows that under mild conditions, the convergence rate of BNGD on OLS is linear. Moreover, close to the optima the linear rate of convergence can be shown to be faster than the best-case linear convergence rate of GD. This offers a concrete result that shows that BNGD could out-perform GD, even if the latter is perfectly-tuned.

    % Now, let us consider the convergence rate of BNGD when it converges to a minimizer. It is mentioned that the effective step size is important to the convergence. Besides the basic inequality (\ref{eq:convergence_rate}), we have the following local accelerated version of convergence rate.

    % The inequality (\ref{eq:convergence_rate}) shows that the convergence of $e_k$ (and hence the loss function, see Lemma~\ref{lemma:convergence_loss}) is linear provided $\hat \varepsilon_k \in (\delta,2/\lambda_{max}-\delta)$ for some positive number $\delta$. In fact, if we enforce $\hat\varepsilon_k = 1/\lambda_{max}$ for each $k$, which is done in the analysis in~\cite{Kohler2018Towards}, then one immediately obtains the same linear convergence rate. But this requires knowledge of $\lambda_{max}$ (problem-dependent) and a modification the BNGD algorithm. We instead focus our analysis on the original BNGD algorithm.

    \begin{theorem}[Convergence rate]
        \label{th:convergence_rate_BNGD}
        If $(a_k,w_k)$ converges to a minimizer with $\hat\varepsilon:=\lim\limits_{k \to \infty} \hat\varepsilon_k < \varepsilon_{max}^*:= 2/\lambda^*_{max}$, then the convergence is linear. Furthermore, when $(a_k,w_k)$ is close to a minimizer, such that $\tfrac{\lambda_{max}\varepsilon |a_k|}{\sigma_k^2}\|e_k\|_H \le \delta <1 $ (this must happen for large enough $k$, since we assumed convergence to a minimizer), then we have
        \begin{align}\label{eq:converge_rho_star}
        \|e_{k+1}\|_H
        \le
         %\min \{
             \tfrac{\rho^*(I - \hat\varepsilon_k H^*) +\delta}{1-\delta}
        %      ,\rho(I - \hat\varepsilon_k H )
        %  \}
        \|e_k\|_H,
    \end{align}
    where $\rho^*(I - \hat\varepsilon_k H ) := \max \{ |1-\hat\varepsilon_k\lambda_{min}^*|, |1-\hat\varepsilon_k\lambda_{max}^*|\}$.
    \end{theorem}
    This statement is proved in appendix Lemma~\ref{lemma:converge_rho_star}. Recall that $H^*,\lambda_{max}^*$ are defined in section \ref{sec:background}. The assumption $\hat\varepsilon<\varepsilon_{max}^*$ is mild since one can prove the set of initial values $(a_0, w_0)$ such that  $(a_k, w_k)$ converges to a minimizer $(a^*, w^*)$ with $\hat\varepsilon > \varepsilon_{max}^*$ and $\det(I-\hat\varepsilon H^*) \neq 0$ is of measure zero (see appendix Lemma \ref{lemma:BN_asto_linear}).

    The inequality (\ref{eq:converge_rho_star}) is motivated by the linearized system corresponding to Eq.~(\ref{eq:BNGD_a})-(\ref{eq:BNGD_w}) near a minimizer. When the iteration converges to a minimizer, the limiting $\hat\varepsilon$ must be a positive number where the assumption $\hat\varepsilon<\varepsilon_{max}^*$ makes sure the coefficient in Eq.~(\ref{eq:converge_rho_star}) is smaller than 1. This implies linear convergence of $\|e_k\|_H$.
    Generally, the matrix $H^*$ has better spectral properties than $H$, in the sense that $\rho^*(I - \hat\varepsilon_k H^*) \le \rho(I - \hat\varepsilon_k H)$, provided $\hat\varepsilon_k >0$,
    where the inequality is strict for almost all $u$. This is a consequence of the Cauchy eigenvalue interlacing property, which one can show directly using mini-max properties of eigenvalues (see appendix Lemma~\ref{lemma:condH}).
    This leads to acceleration effects of BNGD: When $\|e_k\|_H$ is small, the contraction coefficient $\rho$ in Eq.~(\ref{eq:convergence_rate}) can be improved to a lower coefficient in Eq.~(\ref{eq:converge_rho_star}). This acceleration could be significant when $\kappa^*$ is much smaller than $\kappa$, which can happen if the spectral gap of $H$ is very large.

    The acceleration effect can be understood heuristically as follows: due to the over-parameterization introduced by BN, the convergence rate near a minimizer is governed by $H^*$ instead of $H$.
    The former has a degenerate direction $\{ \lambda u : \lambda \in \mathbb{R} \}$, which coincides with the degenerate global minima. Hence, the effective condition number governing convergence is dependent on the largest and the second smallest eigenvalue of $H^*$ (the smallest being 0 in the degenerate minima direction). One can contrast this with the GD case where the smallest eigenvalue of $H$ is considered instead since no degenerate directions exists.
    %Hence, the estimate (\ref{eq:converge_rho_star}) indicates that the optimal BNGD could have a faster convergence rate than the optimal GD, especially when $\kappa^*$ is much smaller than $\kappa$.

    \subsection{Robustness and Acceleration Due to Learning Rate Insensitivity}
    \label{sec:robustness}

    Let us now discuss another advantage BNGD possesses over GD, related to the insensitive dependence of the effective learning rate $\hat \varepsilon_k$ (and by extension, the effective convergence rate in Eq.~(\ref{eq:convergence_rate}) or Eq.~(\ref{eq:converge_rho_star})) on $\varepsilon$. The explicit dependence of $\hat\varepsilon_k$ on $\varepsilon$ is quite complex, but we can give the following asymptotic estimates (see appendix \ref{sec:estimate_hat_eps} for proof).
    \begin{proposition}
    \label{prop:estimate_stepsize}
    Suppose $\varepsilon_a \in (0,1], a_0w_0^Tg>0$, and $||g||^2 \ge \tfrac{w_0^Tg}{\sigma_0^2} g^THw_0$, then
    \begin{itemize}
        \item[(1)] When $\varepsilon$ is small enough, $\varepsilon \ll 1$, the effective step size $\hat\varepsilon_k$ has a same order with $\varepsilon$.% i.e.\,there are two positive constants, $C_1,C_2$, independent on $\varepsilon$ and $k$, such that $C_1 \le \hat\varepsilon_k/\varepsilon \le C_2$.
        \item[(2)] When $\varepsilon$ is large enough, $\varepsilon \gg 1$, the effective step size $\hat\varepsilon_k$ has order $O(\varepsilon^{-1})$.% i.e.\,there are two positive constants, $C_1,C_2$, independent on $\varepsilon$ and $k$, such that $C_1 \le \hat\varepsilon_k\varepsilon \le C_2$.
    \end{itemize}
    \end{proposition}
    Observe that for finite $k$, $\hat \varepsilon_k $ is a differentiable function of $\varepsilon$. Therefore, the above result implies, via the mean value theorem, the existence of some $\varepsilon_0 > 0$ such that $d\hat \varepsilon_k / d\varepsilon |_{\varepsilon = \varepsilon_0}= 0$. Consequently, there is at least some small interval of the choice of learning rates $\varepsilon$ where the performance of BNGD is insensitive to this choice.

    In fact, empirically this is one commonly observed advantage of BNGD over GD, where the former typically allows for a variety of (large) learning rates to be used without adversely affecting performance. The same is not true for GD, where the convergence rate depends sensitively on the choice of learning rate. We will see later in Section~\ref{sec:experiments} that although we only have a local insensitivity result above, the interval of this insensitivity is actually quite large in practice.

    Furthermore, with some additional assumptions and approximations, the explicit dependence of $\hat\varepsilon_k$ on $\epsilon$ can be characterized in a quantitative manner.
    Concretely, we quantify the insensitivity of step size characterized by the interval in which the $\hat\varepsilon$ is close to the optimal step size $\varepsilon_{opt}$ (or the maximal allowed step size $\varepsilon_{max}$ in GD, since $\varepsilon_{opt}$ is very close to $\varepsilon_{max}$ when $\kappa$ is large). Proposition~\ref{prop:estimate_stepsize} indicates that this interval is approximately $[C_1 \varepsilon_{max}, \tfrac{C_2}{\varepsilon_{max}}]$, which crosses a magnitude of $\tfrac{C_2}{C_1 \varepsilon_{max}^2}$, where $C_1, C_2$ are positive constants.

    We set $\varepsilon_a=1, a_0={w_0^Tg}/{\sigma_0}$ (which is the value in the second step if we set $a_0=0$), $\|w_0\|=\|u\|=1$, where Theorem~\ref{th:convergence_BNGD} gives the linear converge result for almost all initial values and the convergence rate can be quantified by the limiting effective learning rate $\hat\varepsilon:=\lim\limits_{k\to\infty}\hat\varepsilon_k = \tfrac{\varepsilon}{\|w_\infty\|^2}$. Consequently, we need to estimate the magnitude $\|w_\infty\|^2$.
    The BNGD iteration implies the following equality,
    % \begin{align}\label{eq:iter_normw}
    %     \|w_{k+1}\|^2 &=\|w_k\|^2  +
    % \tfrac{\varepsilon^2}{\|w_k\|^2} %\big(\tfrac{w_{k-1}^T g}{\sigma_k}\big)^2
    % \tfrac{a_k^2\|w_k\|^2}{\sigma_k^2}
    % \big\| e_k\big\|_{H^2}^2 \nonumber\\
    % &=:
    % \|w_k\|^2  +
    % \tfrac{\varepsilon^2}{\|w_k\|^2} \beta_k ,
    % \end{align}
    \begin{align}\label{eq:iter_normw}
        \|w_{k+1}\|^2 &=
    \|w_k\|^2  +
    \tfrac{\varepsilon^2}{\|w_k\|^2} \beta_k ,
    \end{align}
    where $\beta_k$ is defined as
    $\beta_k:=\tfrac{a_k^2\|w_k\|^2}{\sigma_k^2}
    \big\| e_k\big\|_{H^2}^2$.
    The earlier convergence results motivate the following plausible approximation:
    we assume $\beta_k$ linearly converges to zero and the iteration of $\|w_k\|^2$ can be approximated by $\xi(k+1)$ which obeys the following ODE (whose discretization formally matches Eq.~\eqref{eq:iter_normw}, assuming the aforementioned convergence rate $\rho$):
    \begin{align}
        \xi(0)=\|w_1\|^2, \qquad \dot \xi(t)  = \tfrac{\varepsilon^2\beta_0 \rho^{2t}}{\xi(t)}.
    \end{align}
    Its solution is $\xi^2(t) = \xi^2(0) + \tfrac{\varepsilon^2 \beta_0}{|\ln \rho|}(1-\rho^{2t})$,
    where $\rho\in(0,1)$ depends on $\varepsilon$ and is self-consistently determined by the limiting effective step size, ~i.e. $\rho$ is the spectral radius of $I-\tfrac{\varepsilon}{\xi(\infty)}H$ and $\xi(\infty)$ in turn depends on $\rho$.
    Analyzing the dependence of $\xi(\infty)$ on $\varepsilon$ can give an estimate of the insensitivity interval, which is now $[\varepsilon_{max}, \tfrac1{\beta_0 \varepsilon_{max}}]$, since $\hat\varepsilon \approx \varepsilon$ when $\varepsilon\ll 1$, and $\hat\varepsilon \approx \tfrac{1}{\beta_0\varepsilon}$ when $\varepsilon\gg 1$. (see appendix \ref{sec:estimate_hat_eps}.) Therefore, the magnitude of the interval of insensitivity varies inversely with $\beta_0$. Below, we quantify this magnitude in an average sense.
    \begin{definition}\label{def:Omega}
        The average magnitude of the insensitivity interval of BNGD with $\varepsilon_a=1,a_0=\tfrac{w_0^Tg}{\sigma_0}$ (or $a_0=0$) is defined as $\Omega_H = {1}/{(\bar\beta_H\varepsilon_{max}^2)}$, where $\bar\beta_H$ is the geometric average of $\beta_0$ over $w_0$ and $u$, which we take to be independent and uniformly on the unit sphere $\mathbb{S}^{d-1}$,
        \begin{align}\label{eq:definition_bar_beta}
            \bar\beta_H
            :=
            \exp\big(
                \mathbb{E}_{w_0,u}
                \ln\big[ \big(\tfrac{w_0^T Hu}{w_0^THw_0}\big)^2
                \big\| e_0\big\|_{H^2}^2 \big]
                \big).
            %\le
            %\mathbb{E}_{w_0,u} \big[ \big(\tfrac{w_0^T Hu}{w_0^THw_0}\big)^2 \big\| e_0\big\|_{H^2}^2 \big].
        \end{align}
    \end{definition}
    Note that we use the geometric average rather than the arithmetic average because we are measuring a ratio. Although we can not calculate the value of $\Omega_H$ analytically, we have the following lower bound (see appendix \ref{sec:omega}):
    \begin{proposition}
        \label{prop:estimate_omega}
        For positive definite matrix $H$ with minimal and maximal eigenvalues $\lambda_{min}$ and $\lambda_{max}$ respectively, the $\Omega_H$ defined in Definition \ref{def:Omega} satisfies $\Omega_H \ge \tfrac{d}{C}$, where
        \begin{align}\label{eq:const_C_over}
            C:= 4 \tfrac{Tr[H^2]}{d\lambda^2_{min} } \tfrac{Tr[H]}{d\lambda_{max}} \exp\big(\tfrac{2\ln \kappa}{\kappa-1} (1-\tfrac{Tr[H]}{d\lambda_{min}})\big),% = O(1/d),\\
            %\Omega_H &= \tfrac{\lambda^2_{max}(H)}{4 \bar\beta_H} \ge O(d),
        \end{align}
        $\kappa=\tfrac{\lambda_{max}}{\lambda_{min}}$ is the condition number of $H$.
    \end{proposition}
    As a consequence of the above, if the eigenvalues of $H$ are sampled from a given continuous distribution on $[\lambda_{min},\lambda_{max}]$, such as the uniform distribution, then by law of large numbers, $\Omega_H = O(d)$ for large $d$. This result suggests that the magnitude of the interval on which the performance of BNGD is insensitive to the choice of the learning rate increases linearly in dimension, implying that this robustness effect of BNGD is especially useful for high dimensional problems. Interestingly, although we only derived this result for the OLS problem, this linear scaling of insensitivity interval is also observed in neural networks experiments, where we varied the dimension by adjusting the width of the hidden layers. See Section \ref{sec:experiments_d}.

    The insensitivity to learning rate choices can also lead to acceleration effects if one have to use the same learning rate for training weights with different effective conditioning. This may arise in deep learning applications where each layer's gradient magnitude varies widely, thus requiring very different learning rates to achieve good performance. In this case, BNGD's large range of learning rate insensitivity allows one to use common values across all layers without adversely affecting the performance. This is again in contrast to GD, where such insensitivity is not present. See Section \ref{sec:experiments_nn} for some experimental validation of this claim.

    \section{Experiments}
    \label{sec:experiments}

    Let us first summarize our key findings and insights from the analysis of BNGD on the OLS problem.
    \begin{enumerate}
        \item A scaling law governs BNGD, where certain configurations can be deemed equivalent.
        \item BNGD converges for any learning rate $\varepsilon>0$, provided that $\varepsilon_a \in (0,1]$. In particular, different learning rates can be used for the BN variables $(a)$ compared with the remaining trainable variables $(w)$.
        \item There exists intervals of $\varepsilon$ for which the performance of BNGD is not sensitive to the choice of $\varepsilon$, and the magnitude of this interval grows with dimension.
    \end{enumerate}
    In the subsequent sections, we first validate numerically these claims on the OLS model, and then show that these insights go beyond the simple OLS model we considered in the theoretical framework. In fact, much of the uncovered properties are observed in general applications of BNGD in deep learning.

\subsection{Experiments on OLS}
\label{sec:experiment_OLS}

    Here we test the convergence and stability of BNGD for the OLS model.
    Consider a diagonal matrix $H = \text{diag}(h)$ where $h = (1,...,\kappa)$ is a increasing sequence. The scaling property (Proposition~\ref{prop:scaling}) allows us to set the initial value $w_0$ having same norm with $u$, $\|w_0\| = \|u\| = 1$. Of course, one can verify that the scaling property holds strictly in this case.
    % Note that varying the norm of $w_0$ is same as decoupling step size for $w$ and $a$.
    % \recheck{Is this last sentence useful elsewhere?}
    % (More examples are given in appendix)

    Figure~\ref{fig:100d_logvarcond} gives examples of $H$ with different condition numbers $\kappa$. We tested the loss function of BNGD, compared with the optimal GD (i.e.\,GD with the optimal step size $\varepsilon_{opt}$), in a large range of step sizes $\varepsilon_a$ and $\varepsilon$, and with different initial values of $a_0$.  Another quantity we observe is the effective step size $\hat\varepsilon_k$ of BN. The results are encoded by four different colors: whether $\hat\varepsilon_k$ is close to the optimal step size $\varepsilon_{opt}$, and whether loss of BNGD is less than the optimal GD.
    The results indicate that the optimal convergence rate of BNGD can be better than GD in some configurations, consistent with the statement of Theorem \ref{th:convergence_rate_BNGD}. Recall that this acceleration phenomenon is ascribed to the conditioning of $H^*$ which is better than $H$.
    %This advantage of BNGD is significant when the effective condition number discrepancy between $H$ and $H^*$ is large. However, if this difference is small, the acceleration is imperceptible. This is consistent with our analysis in section~\ref{subsec:convergence_rate}.

    \begin{figure}[thb!]
        \center
        \includegraphics[width=8.cm]{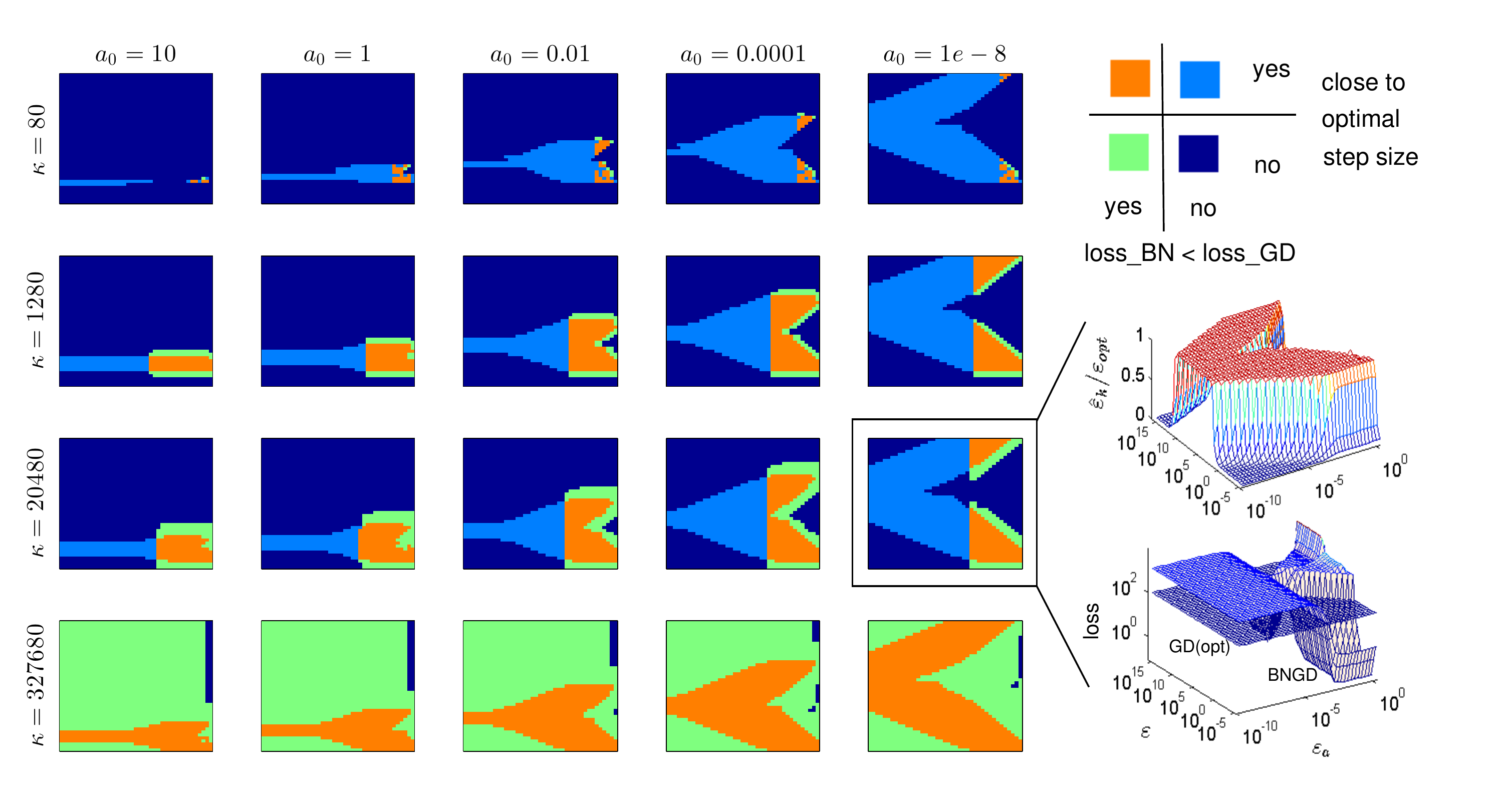}\\
        \caption{Comparison of BNGD and GD on OLS model. The results are encoded by four different colors: whether $\hat\varepsilon_k$ is close to the optimal step size $\varepsilon_{opt}$ of GD, characterized by the inequality $0.8 \varepsilon_{opt} <\hat\varepsilon_k<\varepsilon_{opt}/0.8$, and whether loss of BNGD is less than the optimal GD. Parameters: $H=$ diag(logspace(0,log10($\kappa$),100)), $u$ is randomly chosen uniformly from the unit sphere in $\mathbb{R}^{100}$, $w_0$ is set to $Hu/\|Hu\|$. The GD and BNGD iterations are executed for $k=2000$ steps with the same $w_0$.
        In each image, the range of $\varepsilon_a$ (x-axis) is 1.99 * logspace(-10,0,41), and the range of $\varepsilon$ (y-axis) is logspace(-5,16,43).
        Observe that the performance of BNGD is less sensitive to the condition number, and its advantage is more pronounced when the latter is big.
        }
        \label{fig:100d_logvarcond}
      \end{figure}

    Another important observation is a region such that $\hat\varepsilon$ is close to $\varepsilon_{opt}$, in other words, BNGD significantly extends the range of ``optimal'' step sizes. Consequently, we can choose step sizes in BNGD at greater liberty to obtain almost the same or better convergence rate than the optimal GD.
    However, the size of this region is inversely dependent on the initial condition $a_0$. Hence, this suggests that small $a_0$ at first steps may improve robustness. On the other hand, small $\varepsilon_a$ will weaken the performance of BN. The phenomenon suggests that improper initialization of the BN parameters weakens the power of BN. This experience is encountered in practice, such as~\citep{Cooijmans2016Recurrent}, where higher initial values of BN parameter are detrimental to the optimization of RNN models.

\subsection{Experiments on the Effect of Dimension}
\label{sec:experiments_d}

    In order to validate the approximate results in Section \ref{sec:robustness}, we compute numerically the dependence of the performance of BNGD on the choice of the learning rate $\varepsilon$. Observe from Figure \ref{fig:dvar_Omega} that the quantitative predictions of $\Omega$ in Definition \ref{def:Omega} is consistent with numerical experiments, and the linear-in-dimension scaling of the magnitude of the insensitivity interval is observed. Perhaps more interestingly, the same scaling is also observed in (stochastic) BNGD on fully connected neural networks trained on the MNIST dataset. This suggests that this scaling is relevant, at least qualitatively, beyond the regimes considered in the theoretical parts of this paper.

    % New tests of BNGD on OLS model with step size $\varepsilon_a=1, a_0=0$. Note that the geometric average is used rather than the arithmetic mean. The results confirm that the width of optimal step size mainly depends on the dimension. However, the performance is also effected by the spectrum of $H$. Although small enough step size and large enough step size could have the same effective step size, but in this case the large step size is better because it's optimization course ...

    % The dependence of width $\Omega$ on the dimension $d$....

    \begin{figure}[thb!]
        \center
        \includegraphics[width=3.6cm]{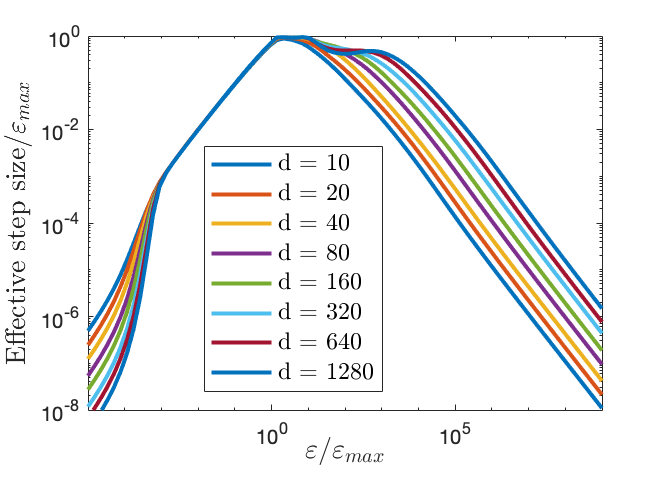}
        \includegraphics[width=3.6cm]{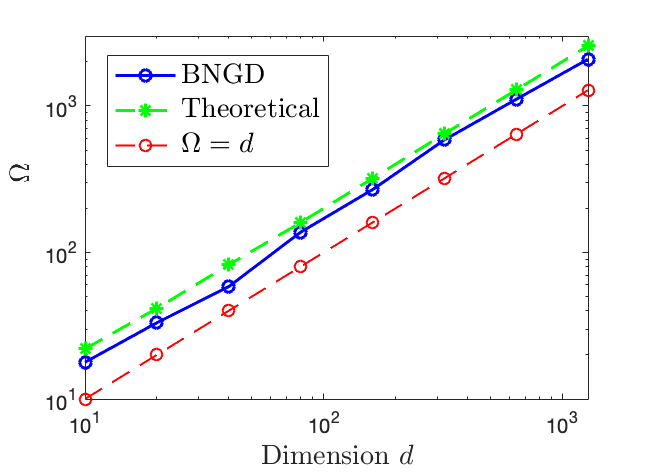}\\
        \includegraphics[width=3.6cm]{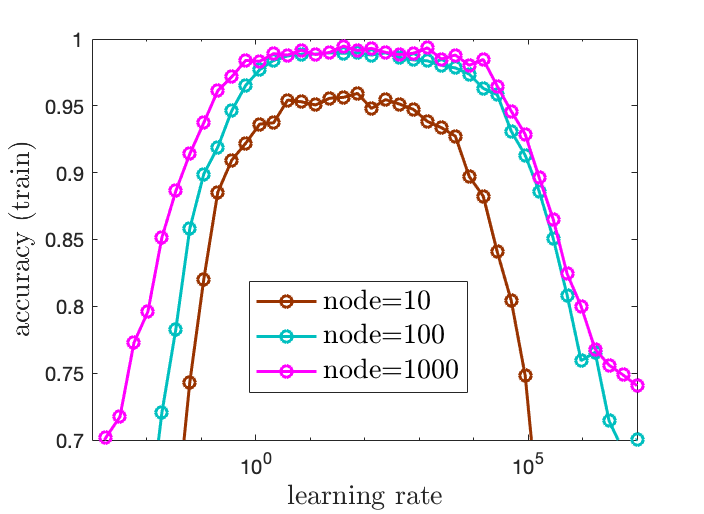}
        \includegraphics[width=3.6cm]{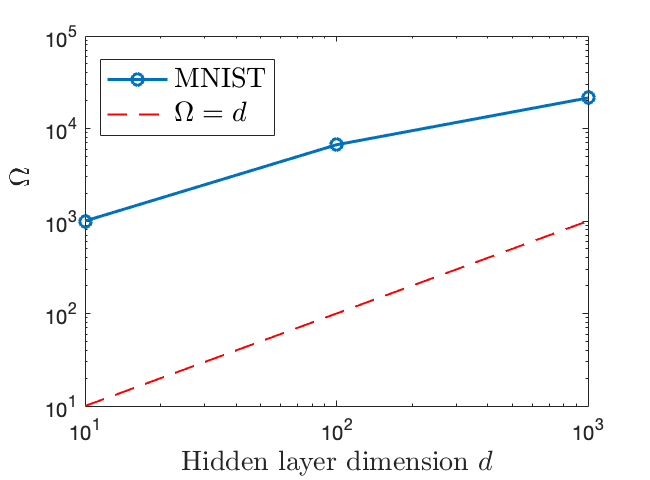}\\
        \caption{Effect of dimension. (Top line) Tests of BNGD on OLS model with step size $\varepsilon_a=1, a_0=0$. Parameters: $H$=diag(linspace(1,10000,d)), $u$ and $w_0$ is randomly chosen uniformly from the unit sphere in $\mathbb{R}^{d}$. The BNGD iterations are executed for $k=5000$ steps. The values are averaged over 500 independent runs.
        (Bottom line) Tests of stochastic BNGD on MNIST dataset, fully connected neural network with one hidden layer and softmax mean-square loss. The separated learning rate for BN Parameters is lr\_a=10, The performance is characterized by the accuracy at the first epoch (averaged over 10 independent runs). The magnitude $\Omega$ is approximately measured for reference.}
        \label{fig:dvar_Omega}
    \end{figure}

% \subsection{Further Deep Learning Experiments}
\subsection{Further Neural Network Experiments}
\label{sec:experiments_nn}

    We conduct further experiments on deep learning applied to standard classification datasets: MNIST~\citep{LeCun1998Gradient}, Fashion MNIST~\citep{Xiao2017Fashion} and CIFAR-10~\citep{Krizhevsky2009Learning}. The goal is to explore if the other key findings outlined at the beginning of this section continue to hold for more general settings.
    For the MNIST and Fashion MNIST dataset, we use two different networks:
    (1) a one-layer fully connected network (784 $\times$ 10) with softmax mean-square loss;
    (2) a four-layer convolution network (Conv-MaxPool-Conv-MaxPool-FC-FC) with ReLU
    activation function and cross-entropy loss.
    For the CIFAR-10 dataset, we use a five-layer convolution network (Conv-MaxPool-Conv-MaxPool-FC-FC-FC). All the trainable parameters are randomly initialized by the Glorot scheme \citep{Glorot2010Understanding} before training. For all three datasets, we use a minibatch size of 100 for computing stochastic gradients.
    In the BNGD experiments, batch normalization is performed on all layers, the BN parameters are initialized to transform the input to zero mean/unit variance distributions, and a small regularization parameter $\epsilon=$1e-3 is added to variance $\sqrt{\sigma^2 +\epsilon}$ to avoid division by zero.

    \textbf{Scaling property}
    Theoretically, the scaling property~\ref{prop:scaling} holds for any layer using BN. However, it may be slightly biased by the regularization parameter $\epsilon$. Here, we test the scaling property in practical settings.
    Figure~\ref{fig:mnist_m1_bnsp_lra1_scaling} gives the loss and accuracy of network-(2) (2CNN+2FC) at the first epoch with different learning rate. The norm of all weights and biases are rescaled by a common factor $\eta$. We observe that the scaling property remains true for relatively large $\eta$. However, when $\eta$ is small, the norm of weights are small. Therefore, the effect of the $\epsilon$-regularization in $\sqrt{\sigma^2 +\epsilon}$ becomes significant, causing the curves to be shifted.
    \begin{figure}[thb!]
      \center
      \includegraphics[width=8.cm]{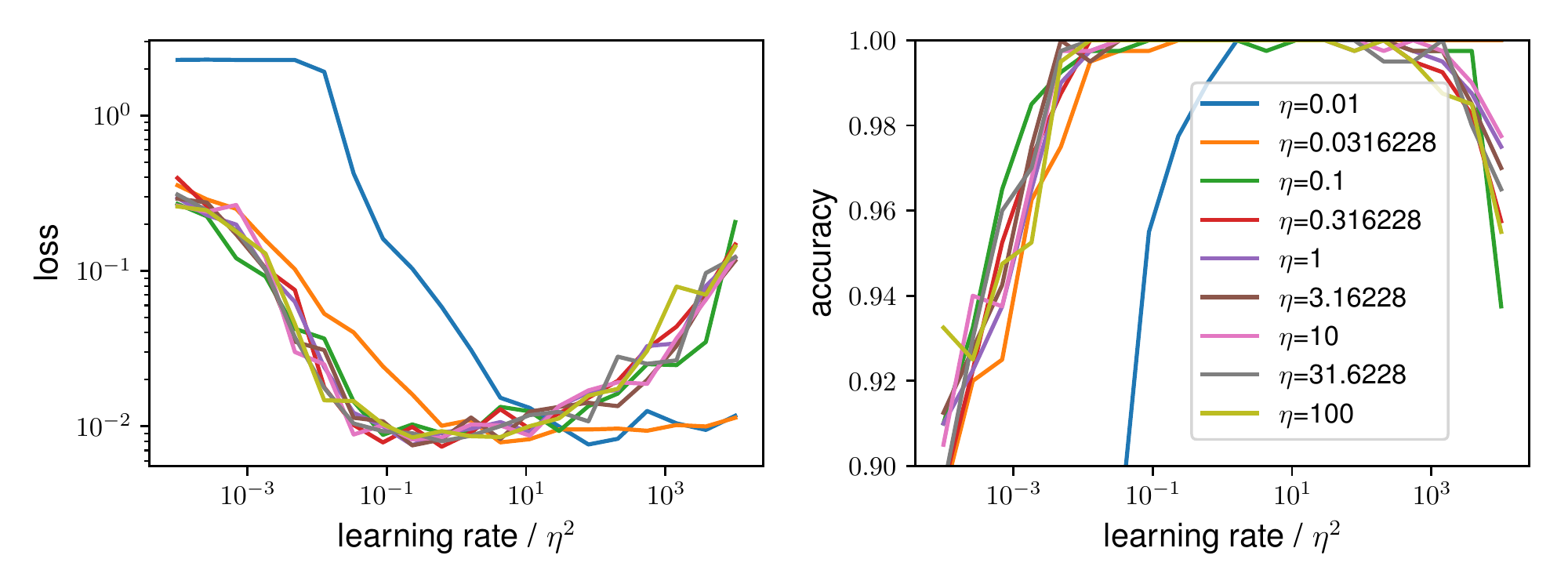}\\
      \caption{Tests of scaling property of the 2CNN+2FC network on MNIST dataset. BN is performed on all layers, and $\epsilon$=1e-3 is added to variance $\sqrt{\sigma^2 +\epsilon}$. All the trainable parameters (except the BN parameters) are randomly initialized by the Glorot scheme, and then multiplied by a same parameter $\eta$.}
      \label{fig:mnist_m1_bnsp_lra1_scaling}
    \end{figure}

    \textbf{Stability for large learning rates}
    We use the loss value at the end of the first epoch to characterize the performance of BNGD and GD methods. Although the training of models have generally not converged at this point, it is enough to extract some relative rate information.
    Figure~\ref{fig:dataset_compare_loss} shows the loss value of the networks on the three datasets. It is observed that GD and BNGD with identical learning rates for weights and BN parameters exhibit a maximum allowed learning rate, beyond which the iterations becomes unstable. On the other hand, BNGD with separate learning rates exhibits a much larger range of stability over learning rate for non-BN parameters, consistent with our theoretical results on OLS problem

    \textbf{Insensitivity of performance to learning rates}
    Observe that BN accelerates convergence more significantly for deep networks, whereas for one-layer networks, the best performance of BNGD and GD are similar. Furthermore, in most cases, the range of optimal learning rates in BNGD is quite large, which is in agreement with the OLS analysis (see Section~\ref{sec:robustness}). This phenomenon is potentially crucial for understanding the acceleration of BNGD in deep neural networks. Heuristically, the ``optimal'' learning rates of GD in distinct layers (depending on some effective notion of ``condition number'') may be vastly different. Hence, GD with a shared learning rate across all layers may not achieve the best convergence rates for all layers at the same time. In this case, it is plausible that the acceleration of BNGD is a result of the decreased sensitivity of its convergence rate on the learning rate parameter over a large range of its choice.

    \begin{figure}[thb!]
      \center
      \includegraphics[width=4cm]{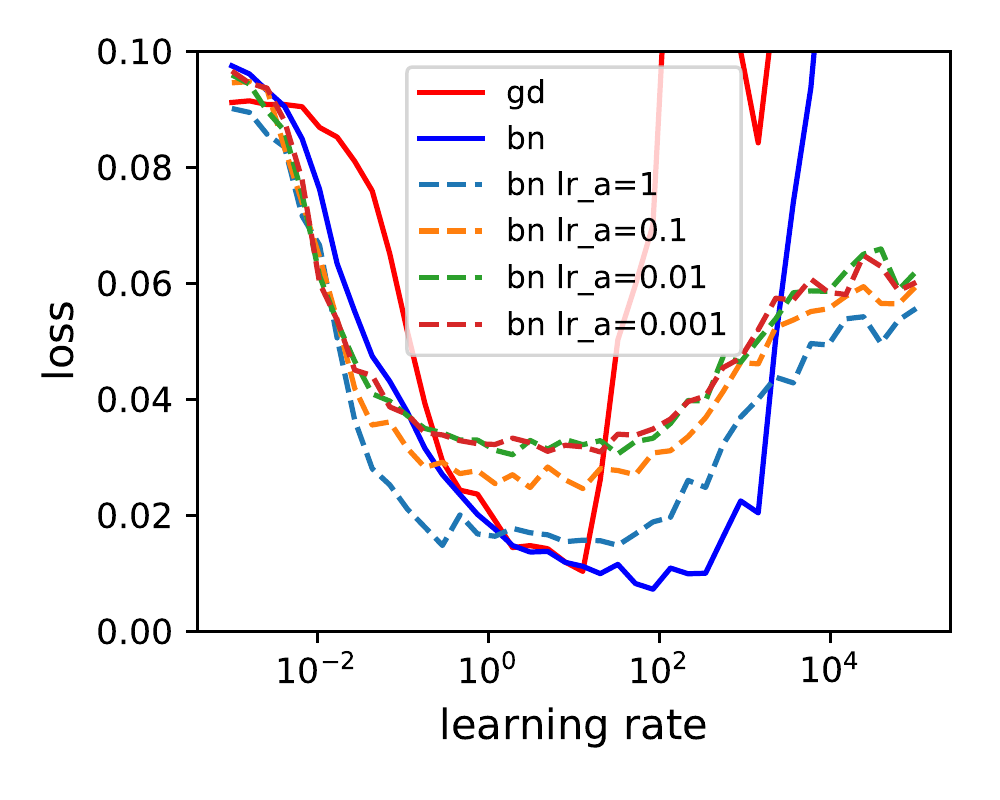}
      ~\includegraphics[width=4cm]{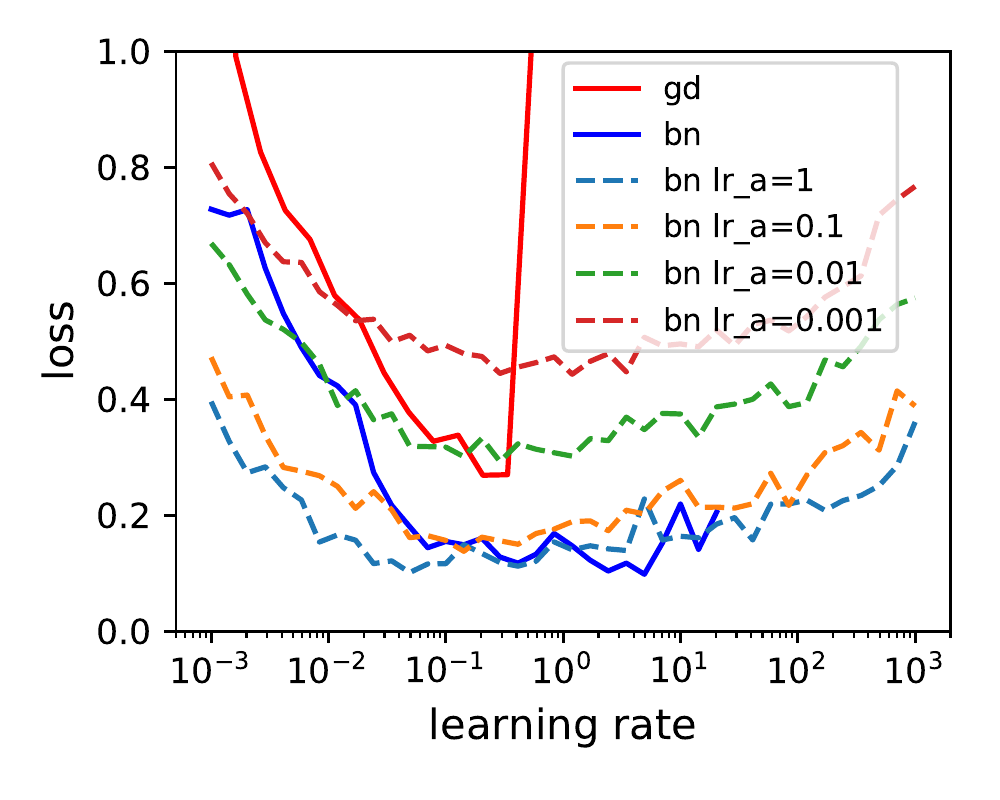}\\
      \includegraphics[width=4cm]{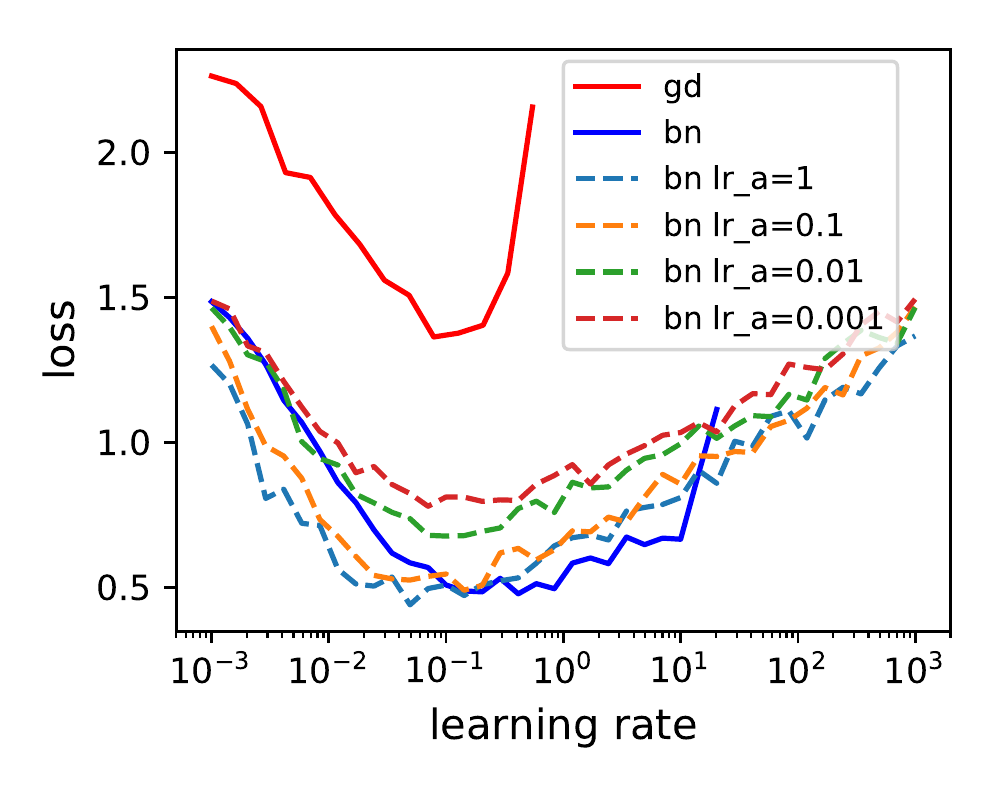}\\
    %   \includegraphics[width=3.5cm]{figures/mnist_m4_compare_loss}\\
    %   ~\includegraphics[width=6.4cm]{figures/fmnist_m1_compare_loss_V2}\\
    %   \includegraphics[width=6.5cm]{figures/cifar10_m1_compare_loss_V2}\\
      \caption{Performance of BNGD and GD method on MNIST (network-(1), 1FC), Fashion MNIST (network-(2), 2CNN+2FC) and CIFAR-10 (2CNN+3FC) datasets. The performance is characterized by the loss value at the first epoch. In the BNGD method, both the shared learning rate schemes and separated learning rate scheme (learning rate lr\_a for BN parameters) are given. The values are averaged over 5 independent runs.}
      \label{fig:dataset_compare_loss}
    \end{figure}

    % \begin{figure}[thb!]
    %     % Requires \usepackage{graphicx}
    %     \center
    %     \includegraphics[width=6.4cm]{figures/mnist_m5_compare_acc}\\
    %     \caption{Effect of network's width when using BNGD method on MNIST (network-(?), 2FC) datasets. The separated learning rate for BN Parameters is lr\_a=10, The performance is characterized by the accuracy value at epoch=1. The values are averaged over 10 independent runs. The bold lines are given for reference.}
    %     \label{fig:mnist_m5_compare_acc}
    % \end{figure}

    \section{Conclusion}
    \label{sec:conclusion}

    In this paper, we analyzed the dynamical properties of batch normalization on OLS, chosen for its simplicity and the availability of precise characterizations of GD dynamics. Even in such a simple setting, we saw that BNGD exhibits interesting non-trivial behavior, including scaling laws, robust convergence properties, acceleration, as well as the insensitivity of performance to the choice of learning rates.
    At least in the setting considered here, our analysis allows one to concretely answer the question of why BNGD can achieve better performance than GD. Although these results are derived only for the OLS model, we show via experiments that these are qualitatively, and sometimes quantitatively valid for more general scenarios. These point to promising future directions towards uncovering the dynamical effect of batch normalization in deep learning and beyond.

\bibliography{icml2019_conference}
\bibliographystyle{icml2019}

\onecolumn
\newpage

% \begin{appendices}
%     \section{ctionFoo}

%     % \chapter{Foobar}

%     \end{appendices}

\appendix

% %% new page
% \mbox{}
% \newpage
% \clearpage

{\Large
\begin{center}
    A Quantitative Analysis of the Effect of Batch Normalization on Gradient Descent\\
    Appendix
\end{center}
}
{
\section{Batch and more general normalization on general objective functions}
\label{sec:general_obj}

Here we consider the generalized versions of batch normalization on general problems, including but not limited to deep neural networks. Consider a smooth loss function $J_0(w_1, ..., w_m)$ and its normalized version $ J(\gamma_1,...,\gamma_m, w_1,...,w_m)$,
\begin{align}\label{eq:loss_generality}
    J(\gamma_1,...,\gamma_m, w_1,...,w_m)
=
J_0 \big( \gamma_1\tfrac{w_1}{\|w_1\|_{S_1}},...,\gamma_m\tfrac{w_m}{\|w_m\|_{S_m}} \big), w_i \neq 0, i=1,...,m.
\end{align}
Here the normalizing matrices $S_i, i=1,...,m$, are assumed to be positive definite and $S_i$ does not depend on $w_i$ and $\gamma_i$ (it could depend on $w_j$ or $\gamma_j, j<i$). For neural networks, choosing $S_i = I$ as the identity matrix, one gets the weight normalization \cite{Salimans2016Weight}. Choosing $S_i$ as the covariance matrix $\Sigma_i$ of $i$th layer output $z_i$, one gets batch normalization. When the covariance matrix is degenerate, one can set $S_i = \Sigma_i + S_0$ with $S_0$ being small but positive definite,\,e.g. $S_0=0.001 I$.

It is obvious that the normalization changes the landscape of the original loss function $J_0$, such as introducing new stationary points which are not stationary points of $J_0$. However, we will show the newly introduced stationary points are strict saddle points and hence can be avoid by many optimization schemes~\cite{Lee2016Gradient,Panageas2017Gradient}.

\subsection{Normalization only introduces strict saddles}

Let us begin with a simple case where $m=1$ in Eq.~(\ref{eq:loss_generality}),~i.e. $
J(\gamma,w; S)
=
J_0\big( \gamma\tfrac{w}{\|w\|_{S}} \big).
$
In this case, the gradients of $J$ are
\begin{align}
\frac{\partial J}{\partial \gamma}
&=
\nabla J_0\big( \gamma\tfrac{w}{\|w\|_{S}} \big)^T \frac{w}{\|w\|_S},\\
\frac{\partial J}{\partial w}
&=
\frac{\gamma}{\|w\|_S}\big(I-\frac{Sww^T}{\|w\|_S^2}\big) \nabla J_0\big( \gamma\tfrac{w}{\|w\|_{S}} \big).
\end{align}
The stationary points $(\gamma,w)$ of $J$ can be grouped into two parts:
\begin{itemize}
    \item[(1)] $\tilde w:=\tfrac{\gamma w}{\|w\|_S}$ is a stationary point of $J_0$. In this case, $\gamma=\pm\|\tilde w\|_S$.
    \item[(2)] $\tilde w$ is not a stationary point of $J_0$. In this case, $\gamma=0, w^T \nabla J_0(\tilde w)=0$.
\end{itemize}

The stationary points in (2) are ones introduced by normalization, giving the Hessian matrix
\begin{align}\label{eq:hessian_A1}
    A_1 := \left(
      \begin{array}{cc}
        \tfrac{\partial^2 J}{\partial \gamma^2} & \tfrac{\partial^2 J}{\partial \gamma\partial w} \\
        \tfrac{\partial^2 J}{\partial w\partial \gamma} & \tfrac{\partial^2 J}{\partial w^2} \\
      \end{array}
    \right)
    =\left(
      \begin{array}{cc}
        \frac{w^T(\nabla^2 J_0(\tilde w) )w}{\|w\|_S^2} & \frac{1}{\|w\|^2_S} (\nabla J_0(\tilde w))^T \\
        \frac{1}{\|w\|^2_S} \nabla J_0(\tilde w) & 0 \\
      \end{array}
    \right).
\end{align}
Since $\nabla J_0(\tilde w) \neq 0$, the rank of $A_1$ is 2. In fact, the nonzero eigenvalues of $A_1$ are:
$$
\frac{a \pm \sqrt{a^2 + 4 \|b\|^2}}{2},
$$
where $a = \frac{w^T(\nabla^2 J_0)w}{\|w\|_S^2}, b=\frac{1}{\|w\|^2_S} \nabla J_0$. Therefore $A_1$ has a negative eigenvalue, and $(\gamma,w)$ is a strict saddle point.

Let us now consider the case of $m >1$. The normalization-introduced stationary points satisfy $\gamma_i=0, w_i^T \nabla J_0(\tilde w_i)=0.$ The Hessian matrix $A$ at these points always has negative eigenvalues because it has a principal minor like $A_1$ in Eq.~(\ref{eq:hessian_A1}). Thus we have the following lemma:
\begin{lemma}\label{lemma:introducing_strict_saddle}
    If $(\gamma_1,...,\gamma_m,w_1,...,w_m)$ is a stationary point of $J$ but $(\tfrac{\gamma_1 w}{\|w_1\|_{S_1}},...,\tfrac{\gamma_m w_m}{\|w_m\|_{S_m}})$ is not a stationary point of $J_0$, then $(\gamma_1,...,\gamma_m,w_1,...,w_m)$ is a strict saddle point of $J$.
\end{lemma}

\subsection{Scaling property and increasing norm of $w_i$}

When using gradient descent to minimize the loss function (\ref{eq:loss_generality}), we need to specify the numerical parameters including the initial values of $\gamma_i$ and $w_i$, which denoted by $\Gamma_0$ and $W_0$ respectively, and the step size for them, denoted by $\varepsilon_\gamma$ and $\varepsilon$. For simplicity, we use the same $\varepsilon_\gamma$ for all $\gamma_i$ and the same $\varepsilon$ for $w_i$. Due to the fact that the scale of $w_i$ does not effect the loss, we immediately have the scaling properties on the set of numerical parameters, or a \emph{configuration} $\{\Gamma_0, W_0, \varepsilon_\gamma,\varepsilon\}$.

\begin{definition}[Equivalent configuration] Two configurations, $\{\Gamma_0, W_0, \varepsilon_\gamma,\varepsilon\}$ and $\{\Gamma_0', W_0', \varepsilon_\gamma',\varepsilon' \}$, are said to be equivalent if for iterates $\{\Gamma_k,W_k\}$, $\{\Gamma_k',W'_k\}$ following these configurations respectively, there is an invertible linear transformation $T$ and a nonzero constant $t$ such that $W_k' = T W_k, \Gamma_k'=t \Gamma_k$ for all $k$.
\end{definition}

It is easy to check the gradient descent on normalized loss function (\ref{eq:loss_generality}) has the following scaling property.
\begin{proposition}[Scaling property]
    \label{prop:scaling_general}
    %{lemma:general_scaling}
    For any $r\neq 0$, the configurations
        $\{\Gamma_0, W_0, \varepsilon_\gamma,\varepsilon\}$ and
        $\{\Gamma_0, r W_0, \varepsilon_\gamma,r^2 \varepsilon\}$
        are equivalent.
\end{proposition}
\begin{proof}
Gradient descent gives the following iteration:
\begin{align}
    \gamma_{i, k+1} &= \gamma_{i,k} - \varepsilon_\gamma \tfrac{\partial J}{\partial \gamma_i}(\Gamma_k,W_k),\\
    w_{i,k+1} &= w_{i,k} - \varepsilon \tfrac{\partial J}{\partial w_i}(\Gamma_k,W_k).
\end{align}
It is easy to check that
$
    \frac{\partial J}{\partial (r w_i)} =
    \frac{1}{r}
    \frac{\partial J}{\partial w_i},
$
$
    r w_{i,k+1} = r w_{i,k} - r^2\varepsilon \tfrac{\partial J}{\partial (r w_i)}(\Gamma_k,W_k).
$
Let $\gamma_i=\gamma_i', w_i' = r w_i, \varepsilon_\gamma'=\varepsilon_\gamma, \varepsilon' = r^2\varepsilon$, then we immediately have the equivalence result.
\end{proof}

Another consequence of the invariance of loss functions with respect to the scale of $w_i$ is the orthogonality between $w_i$ and $\tfrac{\partial J}{\partial w_i}$. In fact, we have $0=\frac{\partial l}{\partial \|w_i\|} = \frac{w_i}{\|w_i\|} \cdot \frac{\partial J}{\partial w_i}$.
As a consequence, we have the following property.

\begin{proposition}[Increaing norm of $w_i$]
    \label{prop:increasing_norm}
    %{lemma:general_scaling}
    For any configuration $\{\Gamma_0, W_0, \varepsilon_\gamma,\varepsilon\}$, the norm of each $w_i$ is incresing during gradient descent iteration.
\end{proposition}
\begin{proof}
    According to the orthogonality between $w_i$ and $\tfrac{\partial J}{\partial w_i}$, we have
\begin{align}\label{eq:pythagoras}
    \|w_{i,k+1}\|^2 = \|w_{i,k}\|^2 + \varepsilon^2 \big\|\tfrac{\partial J}{\partial w_{i,k}} \big\|^2 \ge \|w_{i,k}\|^2,
\end{align}
which finishes the proof.
\end{proof}

}

{%\color{cyan}
\subsection{Convergence for arbitrary step size}

As a consequence of scaling property and the increasing-norm property, we have the following convergence result, which says that convergence for small learning rates implies convergence for arbitrary learning rates for weights.

\begin{theorem}[Convergence of the gradient descent on (\ref{eq:loss_generality})]
\label{th:convergence_NormGD}
%Suppose the loss function $J$ in (\ref{eq:loss_generality}) is an analytic function.
If there are two positive constants, $\varepsilon_\gamma^*, \varepsilon^*$, such that the gradient descent on $J$ converges for any initial value $\Gamma_0, W_0$ such that $\|w_{i,0}\|=1$ and step size $\varepsilon_\gamma<\varepsilon_\gamma^*,\varepsilon<\varepsilon^*$, then the gradient of $w_i$ converges for arbitrary step size $\varepsilon>0$ and $\varepsilon_\gamma<\varepsilon_\gamma^*$.
\end{theorem}
\begin{proof}
Firstly, the norm of each $w_{i,k}$ must converge for any step size $\varepsilon>0$ and $\varepsilon_\gamma<\varepsilon_\gamma^*$. In fact, if $w_{i,k}$ is not bounded, then there is a $k=K$ such that $\tfrac{\varepsilon}{\|w_{i,K}\|^2} < \varepsilon^*$. Then using the scaling property, one has a configuration contradicts the assumptions.

Secondly, the gradients of $w_i$, $\frac{\partial J}{\partial w_{i,k}}$, converges to zero.  According to Eq.~(\ref{eq:pythagoras}), we have,
\begin{align}%\label{eq:}
    \|w_{i,\infty}\|^2 = \|w_{i,0}\|^2 + \varepsilon^2 \sum_{k=0}^{\infty}\big\|\tfrac{\partial J}{\partial w_{i,k}} \big\|^2
    < \infty
\end{align}
from which it follows by using $\sum_k \tfrac1k = \infty$ that
\begin{align}%\label{eq:}
    \liminf_{k\to\infty} k \big\|\tfrac{\partial J}{\partial w_{i,k}} \big\|^2 = 0.
\end{align}
\end{proof}

}

\section{Proof of Theorems on OLS problem}

\subsection{Gradients and the Hessian matrix}
\label{sec:matrix}

The objective function in OLS problem (\ref{eq:OLS_BN}) has an equivalent form:
\begin{align}\label{eq:BN_loss}
    J(a,w) = \tfrac12 ( u - \tfrac{a}{\sigma} w)^T H ( u - \tfrac{a}{\sigma} w)
    = \tfrac12\|u\|_H^2 - \tfrac{w^Tg}{\sigma}a + \tfrac12 a^2,
\end{align}
where $u = H^{-1}g$.

The gradients are:
\begin{align}%\label{eq:}
    \tfrac{\partial J}{\partial a} &= - \tfrac{1}{\sigma} (w^T H u - \tfrac{a}{\sigma} w^T H w)
        = - \tfrac{1}{\sigma} w^T g + a,\\
    \tfrac{\partial J}{\partial w} &= - \tfrac{a}{\sigma} ( H u - \tfrac{a}{\sigma} Hw)
+ \tfrac{a}{\sigma^3} (w^TH u - \tfrac{a}{\sigma} w^T Hw) Hw
= - \tfrac{a}{\sigma} g + \tfrac{a}{\sigma^3} (w^T g) Hw.
\end{align}

The Hessian matrix is
\begin{align}%\label{eq:}
    \left(
      \begin{array}{cc}
        \tfrac{\partial^2 J}{\partial a^2} & \tfrac{\partial^2 J}{\partial a\partial w} \\
        \tfrac{\partial^2 J}{\partial w\partial a} & \tfrac{\partial^2 J}{\partial w^2} \\
      \end{array}
    \right)
    =\left(
      \begin{array}{cc}
        1 & A_{21}^T \\
        A_{21} & A_{22} \\
      \end{array}
    \right)
\end{align}
where
\begin{align}%\label{eq:}
    A_{22} &= \tfrac{a}{\sigma^3} (w^T g)\Big[
        H + \tfrac{1}{w^T g} \big( (Hw)g^T + g(Hw)^T \big)
        - \tfrac{3}{\sigma^2} (Hw)(Hw)^T \Big], \\
    A_{21} &= -\tfrac{1}{\sigma} \big(g - \tfrac{1}{\sigma^2} (w^T g) Hw \big).
\end{align}

The objective function $J(a,w)$ has saddle points, $\{(a^*, w^*)|a^*=0, w^{*T}g=0\}$. The Hessian matrix at those saddle points has at least one negative eigenvalue, i.e.~the saddle points are strict. In fact, the eigenvalues at the saddle point $(a^*,w^*)$ are
$\Big\{\frac12(1\pm \sqrt{1+ 4 \tfrac{\|g\|^2}{w^{*T}Hw^*}}),0,...,0\Big\}$ which contains $d-2$ repeated zero, a positive and a negative eigenvalue.

On the other hand, the nontrivial critical points satisfies the relations,
\begin{align}%\label{eq:}
    a^* = \pm \sqrt{u^{T} H u},
    w^* \mathbin{\!/\mkern-5mu/\!} u,
\end{align}
where the sign of $a^*$ depends on the direction of $u, w^*$, i.e.\,$sign(a^*) = sign(u^{T}w^*)$. It is easy to check that the nontrivial critical points are global minimizers. The Hessian matrix at those minimizers is $\text{diag}\big(1,\tfrac{\|u\|^2}{\|w^*\|^2} H^*\big)$ where the matrix $H^*$ is
\begin{align}%\label{eq:}
    H^* = H - \tfrac{Huu^TH}{u^THu}
\end{align}
which is positive semi-definite and has a zero eigenvalue with eigenvector $u$, i.e.\,$H^*u=0$. The following lemma, similar to the well-known Cauchy interlacing theorem, gives an estimate of eigenvalues of $H^*$.

\begin{lemma}%[]
\label{lemma:condH}
If $H$ is positive definite and $H^*$ is defined as $H^*=H-\tfrac{Huu^TH}{u^THu}$, then the eigenvalues of $H$ and $H^*$ satisfy the following inequalities:
\begin{align}%\label{eq:}
    0=\lambda_1(H^*) < \lambda_1(H)
    \le \lambda_2(H^*) \le \lambda_2(H)
    \le ... \le
    \lambda_d(H^*) \le \lambda_d(H).
\end{align}
Here $\lambda_i(H)$ means the $i$-th smallest eigenvalue of $H$.
\end{lemma}
\begin{proof}
(1) According to the definition, we have $H^*u=0$, and for any $x \in \mathbb{R}^d$,
\begin{align}%\label{eq:}
    x^TH^*x = x^THx - \tfrac{(x^THu)^2}{u^THu} \in [0, x^THx],
\end{align}
which implies $H^*$ is positive semi-definite, and $\lambda_i(H^*)\ge\lambda_1(H^*)=0$.
Furthermore, we have the following equality:
\begin{align}%\label{eq:}
    x^TH^*x = \min\limits_{t\in\mathbb{R}} \|x-tu\|_H^2.
\end{align}

(2) We will prove $\lambda_i(H^*) \le \lambda_i(H)$ for all $i$, $1\le i\le d$. In fact, using the Min-Max Theorem, we have
\begin{align*}%\label{eq:}
    \lambda_i(H^*)
    =
    \min\limits_{dim V =i}\max\limits_{x\in V} \tfrac{x^TH^*x}{\|x\|^2}
    \le
    \min\limits_{dim V =i}\max\limits_{x\in V} \tfrac{x^THx}{\|x\|^2}
    =
    \lambda_i(H).
\end{align*}

(3) We will prove $\lambda_i(H^*) \ge \lambda_{i-1}(H)$ for all $i$, $2\le i\le d$. In fact, using the Max-Min Theorem, we have
\begin{align*}%\label{eq:}
    \lambda_i(H^*)
    &=
    \max\limits_{dim V =n-i+1}\min\limits_{x\in V} \tfrac{x^TH^*x}{\|x\|^2}
    =
    \max\limits_{dim V =n-i+1, u \perp V}\min\limits_{x\in V}
    \min_{t\in \mathbb{R}} \tfrac{\|x-tu\|_H^2}{\|x\|^2} \\
    &\ge
    \max\limits_{dim V =n-i+1, u\perp V}\min\limits_{x\in V}
    \min_{t\in \mathbb{R}} \tfrac{\|x-tu\|_H^2}{\|x-tu\|^2} \\
    &=
    \max\limits_{dim V =n-i+1}\min\limits_{y\in span\{V,u\}}
    \tfrac{\|y\|_H^2}{\|y\|^2}, y=x-tu \\
    &\ge
    \max\limits_{dim V =n-(i-1)+1}\min\limits_{y\in V} \tfrac{y^THy}{\|y\|^2}
    =
    \lambda_{i-1}(H),
\end{align*}
where we have used the fact that $x\perp u$, $\|x-tu\|^2 = \|x\|^2 + t^2\|u\|^2 \ge \|x\|^2$.
\end{proof}

There are several corollaries related to the spectral property of $H^*$. We first give some definitions.  Since $H^*$ is positive semi-definite, we can define the $H^*$-seminorm.
\begin{definition} The $H^*$-seminorm of a vector $x$ is defined as $\|x\|_{H^*} := x^TH^*x$. $\|x\|_{H^*}=0$ if and only if $x$ is parallel to $u$.
\end{definition}
\begin{definition} The pseudo-condition number of $H^*$ is defined as $\kappa^*(H^*):=\tfrac{\lambda_d(H^*)}{\lambda_2(H^*)}$.
\end{definition}
\begin{definition} For any real number $\varepsilon$, the pseudo-spectral radius of the matrix $I-\varepsilon H^*$ is defined as $\rho^*(I-\varepsilon H^*):=\max\limits_{2\le i\le d} |1-\varepsilon \lambda_i(H^*)|$.
\end{definition}

The following corollaries are direct consequences of Lemma~\ref{lemma:condH}, hence we omit the proofs.

\begin{corollary}
\label{cor:condH_eq}
The pseudo-condition number of $H^*$ is less than or equal to the condition number of $H$ :
\begin{align}\label{eq:cond_H_star}
    \kappa^*(H^*):=\tfrac{\lambda_d(H^*)}{\lambda_2(H^*)}
    \le
    \tfrac{\lambda_d(H)}{\lambda_1(H)} =: \kappa(H),
\end{align}
where the equality holds if and only if $u \perp span\{v_1,v_d\}$, $v_i$ is the eigenvector of $H$ corresponding to the eigenvalue $\lambda_i(H)$.
\end{corollary}

\begin{corollary}
%\label{cor:condH_eq}
For any vector $x \in \mathbb{R}^d$ and any real number $\varepsilon$, we have
$\|(I-\varepsilon H^*)x\|_{H^*} \le \rho^*(I-\varepsilon H^*) \|x\|_{H^*}$.
\end{corollary}

\begin{corollary}

For any positive number $\varepsilon>0$, we have
\begin{align}\label{eq:rho_H_star}
    \rho^*(I-\varepsilon H^*) \le \rho(I-\varepsilon H),
\end{align}
where the inequality is strict if $u^T v_i \neq 0$ for $i=1,d$.
\end{corollary}

It is obvious that the inequality in Eq.~(\ref{eq:cond_H_star}) and Eq.~(\ref{eq:rho_H_star}) is strict for almost all $u$ with respect to the Lebesgue measure. Particularly, if the spectral gap $\lambda_2(H)-\lambda_1(H)$ or $\lambda_d(H)-\lambda_{d-1}(H)$ is large, the condition number $\kappa^*(H^*)$ could be much smaller than $\kappa(H)$.

\subsection{Scaling property}

The dynamical system defined in Eq.~(\ref{eq:BNGD_a})-(\ref{eq:BNGD_w}) is completely determined by a set of configurations $\{H,u,a_0,w_0,\varepsilon_a,\varepsilon\}$.
It is easy to check the system has the following scaling property:
\begin{lemma}[Scaling property]
%\label{prop:scaling}
%\variant{\ref{prop:scaling} revisited}
Suppose $\mu\neq0,\gamma\neq0,r\neq0,Q^TQ=I$, then
\begin{itemize}
    \item[(1)] The configurations $\{\mu Q^THQ,\tfrac{\gamma}{\sqrt{\mu}} Qu, \gamma a_0, \gamma Q w_0, \varepsilon_a, \varepsilon\}$ and $\{H,u,a_0,w_0,\varepsilon_a,\varepsilon\}$ are equivalent.
    \item[(2)] The configurations $\{H,u,a_0,w_0,\varepsilon_a,\varepsilon\}$ and
    $\{H,u,a_0,r w_0,\varepsilon_a,r^2 \varepsilon\}$ are equivalent.
\end{itemize}
\end{lemma}

\subsection{Proof of Theorem~\ref{th:convergence_BNGD}}
\label{sec:proof_3p3}

Recall the BNGD iterations
\begin{align*}\label{eq:app:BNGD}
    a_{k+1} &= a_k + \varepsilon_a \Big( \tfrac{w_k^T g}{\sigma_k} - a_k \Big),\\
    w_{k+1} &= w_k  +  \varepsilon \tfrac{a_k}{\sigma_k} \Big( g- \tfrac{w_k^T g}{\sigma_k^2}  H w_k \Big).
\end{align*}

The scaling property simplify our analysis by allowing us to set, for example, $\|u\|=1$ and $\|w_0\| = 1$. In the rest of this section, we only set $\|u\|=1$.

For the step size of $a$, it is easy to check that $a_k$ tends to infinity with $\varepsilon_a>2$ and initial value $a_0=1,w_0=u$. Hence we only consider $0<\varepsilon_a < 2$, which make the iteration of $a_k$ bounded by some constant $C_a$.

\begin{lemma}[Boundedness of $a_k$]
\label{lemma:boundedness_a}
If the step size $0<\varepsilon_a < 2$, then the sequence $a_k$ is bounded for any $\varepsilon>0$ and any initial value $(a_0,w_0)$.
\end{lemma}
\begin{proof}
Define $\alpha_k:=\tfrac{w_k^T g}{\sigma_k}$, which is bounded by
$|\alpha_k| \le \sqrt{u^THu} =:C$, then
\begin{align*}%\label{eq:}
    a_{k+1} &= (1-\varepsilon_a)a_k + \varepsilon_a \alpha_k \\
    & = (1-\varepsilon_a)^{k+1} a_0 + (1-\varepsilon_a)^k \varepsilon_a \alpha_0 + ...+ (1-\varepsilon_a)\varepsilon_a \alpha_{k-1} + \varepsilon_a \alpha_k.
\end{align*}
Since $|1-\varepsilon_a| < 1$, we have
$|a_{k+1}| \le |a_0| + 2 C \sum_{i=0}^k |1-\varepsilon_a|^i \le |a_0| + 2C\tfrac{1}{1-|1-\varepsilon_a|}$.
\end{proof}

According to the iterations (\ref{eq:app:BNGD}), we have
\begin{align}%\label{eq:}
    u - \tfrac{w_k^Tg}{\sigma_k^2} w_{k+1}
    =
    \Big(I - \varepsilon \tfrac{a_k}{\sigma_k} \tfrac{w_k^Tg}{\sigma_k^2} H \Big)
    \Big(u - \tfrac{w_k^Tg}{\sigma_k^2} w_k\Big).
\end{align}
Define
\begin{align}%\label{eq:}
    e_k &:=  u- \tfrac{w_k^T g}{\sigma_k^2} w_k,\\
    q_k &:= u^THu- \tfrac{(w_k^Tg)^2}{\sigma_k^2}  = \|e_k\|^2_H \ge 0,\\
    \hat \varepsilon_k &:= \varepsilon \tfrac{a_k}{\sigma_k} \tfrac{w_k^Tg}{\sigma_k^2},
\end{align}
and using the property $\tfrac{w^Tg}{\sigma_k^2} = \underset{t}{\operatorname{argmin}} \|u - t w\|_H$, and the property of $H$-norm, we have
%\mathop{\arg\min}_{t}
\begin{align}\label{eq:q_ineq}
    q_{k+1} \le
    \Big|\Big|
    u - \tfrac{w_k^Tg}{\sigma_k^2} w_{k+1}
    \Big|\Big|_H^2
    =
    \|(I - \hat\varepsilon_k H)e_k \|_H^2
    \le
    \rho(I - \hat\varepsilon_k H)^2 q_k.
\end{align}
Therefore we have the following lemma to make sure the iteration converge:
\begin{lemma}
\label{lemma:convergence_basic}
Let $0<\varepsilon_a<2$.
If there are two positive numbers $\varepsilon^-$ and $\hat \varepsilon^+$,
and the effective step size $\hat \varepsilon_k$ satisfies
\begin{align}\label{eq:step_ineq}
    0 < \tfrac{\varepsilon^-}{\|w_k\|^2}
    \le
    \hat \varepsilon_k
    \le
    \hat\varepsilon^+
    <\tfrac{2}{\lambda_{max}}
\end{align}
for all $k$ large enough, then the iterations (\ref{eq:app:BNGD}) converge to a minimizer.
\end{lemma}
\begin{proof}
Without loss of generality, we assume $\tfrac{\varepsilon^-}{\|w_k\|^2} < \tfrac1{\lambda_{max}}$ and the inequality (\ref{eq:step_ineq}) is satisfied for all $k \ge 0$.
We will prove $\|w_k\|$ converges and the direction of $w_k$ converges to the direction of $u$.

(1) Since $\|w_k\|$ is always increasing, we only need to prove it is bounded.
We have,
\begin{align}%\label{eq:}
    \|w_{k+1}\|^2 &= \|w_k\|^2 + \varepsilon^2 \tfrac{a_k^2}{\sigma_k^2} \|He_k\|^2\\
    &=
    \|w_0\|^2 + \varepsilon^2 \sum_{i=0}^k \tfrac{a_i^2}{\sigma_i^2} \|He_i\|^2\\
    & \le
    \|w_0\|^2 + \varepsilon^2 \lambda_{max}\sum_{i=0}^k \tfrac{a_i^2}{\sigma_i^2} q_i\\
    & \le
    \|w_0\|^2 + \varepsilon^2  \tfrac{\lambda_{max} C_a^2}{\lambda_{min}}\sum_{i=0}^k \tfrac{q_i}{\|w_i\|^2}.
    \label{eq:normw_ineq}
\end{align}
The inequality in last lines are based on the fact that
$\|He_i\|^2 \le \lambda_{max}\|e_i\|^2_H$, and $|a_k|$ are bounded by a constant $C_a$. Next, we will prove $\sum_{i=0}^\infty \tfrac{q_i}{\|w_i\|^2} <\infty$, which implies $\|w_k\|$ are bounded.

According to the estimate Eq.~(\ref{eq:q_ineq}), we have
\begin{align}%\label{eq:}
    q_{k+1}
    &\le
    \max_{i}\{|1-\hat\varepsilon^+ \lambda_i|^2,
    |1-\tfrac{\varepsilon^- \lambda_i}{\|w_k\|^2}|^2\}
    q_k\\
    &\le
    \max\{1-\gamma^+,
    1-\tfrac{\varepsilon^- \lambda_{min}}{\|w_k\|^2} \}
    q_k,
\end{align}
where $1-\gamma^+ = \max_i\{ |1-\hat\varepsilon^+ \lambda_i|^2 \} \in (0,1)$. Using the definition of $q_k$, we have
\begin{align}%\label{eq:}
    q_k - q_{k+1}
    &\ge
    \tfrac{ \min\{\gamma^+\|w_0\|^2, \varepsilon^- \lambda_{min}  \} }{\|w_k\|^2}
    q_k =: \tfrac{Cq_k}{\|w_k\|^2}
    \ge 0.
\end{align}
Since $q_k$ is bounded in $[0, u^THu]$, summing both side of the inequality, we get the bound of the infinite series  $\sum\limits_k \tfrac{q_k}{\|w_k\|^2} \le \frac{u^THu}{C} < \infty$.

(2) Since $\|w_k\|$ is bounded, we denote $\hat \varepsilon^- := \tfrac{\varepsilon^-}{\|w_\infty\|^2}$, and define
$\rho:= \max\limits_i\{ |1-\hat\varepsilon^\pm \lambda_i|\} \in (0,1)$, then the inequality (\ref{eq:q_ineq}) implies  $q_{k+1} \le \rho^2 q_k$. As a consequence, $q_k$ tends to zero,
%$\lim\limits_{k \to \infty} q_k = 0$,
which implies the direction of $w_k$ converges to the direction of $u$.

(3) The convergence of $a_k$ is a consequence of $w_k$ converging.

\end{proof}

Since $a_k$ is bounded, we assume $|a_k| < \tilde C_a \sqrt{u^THu}$, $\tilde C_a \ge 1$, and define  $\varepsilon_0 := \tfrac{1}{2 \tilde C_a \kappa \lambda_{max}}$. The following lemma gives the convergence for small step size.

\begin{lemma}
\label{lemma:convergence_BNGD_small}
If the initial values $(a_0,w_0)$ satisfies $a_0w_0^T g >0$, and step size satisfies $\varepsilon_a \in (0,1], \varepsilon/\|w_0\|^2 < \varepsilon_0$, then the sequence $(a_k,w_k)$ converges to a global minimizer.
\end{lemma}

\textbf{Remark 1:}
If we set $a_0=0$, then we have $w_1 = w_0, a_1 = \varepsilon_a \tfrac{w_0^Tg}{\sigma_0}$, hence $a_1 w_1^Tg >0$ provided $w_0^Tg\neq0$.

\textbf{Remark 2:}
For the case of $\varepsilon_a \in (1,2)$, if the initial value satisfies an additional condition $0<|a_0|\le \varepsilon_a \tfrac{|w_0^Tg|}{\sigma_0}$, then we have $(a_k,w_k)$ converging to a global minimizer as well.

\begin{proof}

Without loss of generality, we only consider the case of $a_0>0, w_0^Tg>0, \|w_0\|\ge 1$.

(1)
We will prove $a_k>0, w_k^Tg>0$ for all $k$. Denote $y_k := w_k^T g$, $\delta=\tfrac{\|g\|}{4\kappa}$.

On the one hand, if $a_k>0, 0 < y_k <2 \delta$, then
\begin{align}%\label{eq:}
    y_{k+1} \ge y_k + \varepsilon \tfrac{a_k}{\sigma_k}\tfrac{\|g\|^2}{2} \ge y_k.
\end{align}
On the other hand, when $a_k>0, y_k>0, \varepsilon < \varepsilon_0$, we have
\begin{align}%\label{eq:}
    y_{k+1}
    &\ge
    \varepsilon \tfrac{a_k \|g\|^2}{\sigma_k}
    +
    y_k
    \Big( 1- \varepsilon \tfrac{a_k}{\sigma_k^2} \sqrt{g^THg}\Big)
    \ge
    \tfrac12 y_k,\\
    a_{k+1}
    &\ge
    \min\{a_k, y_k/\sigma_k\}.
\end{align}
As a consequence, we have $a_k>0, y_k \ge \delta_y := \min\{y_0, \delta\}$ for all $k$ by induction.

(2) We will prove the effective step size $\hat \varepsilon_k$ satisfies the condition in Lemma~\ref{lemma:convergence_basic}.

Since $a_k $ is bounded, $\varepsilon<\varepsilon_0$, we have
\begin{align}%\label{eq:}
    \hat \varepsilon_k &:= \varepsilon \tfrac{a_k}{\sigma_k} \tfrac{w_k^Tg}{\sigma_k^2}
    \le \tfrac{\varepsilon \tilde C_a \lambda_{max}}{\lambda_{min}\|w_k\|^2}
    \le \varepsilon \tilde C_a \kappa
    =:
    \hat \varepsilon^+
    < \tfrac{1}{2\lambda_{max}},
\end{align}
and
\begin{align}%\label{eq:}
q_{k+1}
\le
(1-\hat \varepsilon_k \lambda_{min})^2 q_k
\le
(1-\hat \varepsilon_k \lambda_{min}) q_k
< q_k.
\end{align}
which implies $\tfrac{w_{k+1}^Tg}{\sigma_{k+1}} \ge \tfrac{w_k^Tg}{\sigma_k} \ge \tfrac{w_0^Tg}{\sigma_0}$.
Furthermore, we have $a_k \ge \min\{a_0, \tfrac{w_0^Tg}{\sigma_0}\}$,
and there is a positive constant $\varepsilon^->0$ such that
\begin{align}%\label{eq:}
    \hat \varepsilon_k
    \ge
    \varepsilon \tfrac{a_k}{\lambda_{max}\|w_k\|^2} \tfrac{w_k^Tg}{\sigma_k} \ge \tfrac{\varepsilon^-}{\|w_k\|^2}.
\end{align}

(3) Employing the Lemma~\ref{lemma:convergence_basic}, we conclude that $(a_k, w_k)$ converges to a global minimizer.
\end{proof}

\begin{lemma}
\label{lemma:convergence_BNGD_small2}
If step size satisfies $\varepsilon_a \in (0,1], \varepsilon/\|w_0\|^2 < \varepsilon_0$, then the sequence $(a_k,w_k)$ converges.
\end{lemma}
\begin{proof}
Thanks to Lemma~\ref{lemma:convergence_BNGD_small}, we only need to consider the case of $a_k w_k^Tg \le 0$ for all $k$, and we will prove the iteration converges to a saddle point in this case. Since the case of $a_k=0$ or $w_k^Tg=0$ is trivial, we assume $a_k w_k^Tg < 0$ below. More specifically , we will prove $|a_{k+1}| < r |a_k|$ for some constant $r \in (0,1)$, which implies convergence to a saddle point.

(1) If $a_k$ and $a_{k+1}$ have a same sign, hence different sign with $w_k^Tg$, then we have
$|a_{k+1}| = |1-\varepsilon_a\|a_k| - \varepsilon_a|w_k^Tg|/\sigma_k \le |1-\varepsilon_a\|a_k|$.

(2) If $a_k$ and $a_{k+1}$ have different signs, then we have
\begin{align}%\label{eq:}
    \tfrac{|w_k^Tg|}{|a_k\sigma_k|}
    \le
    \varepsilon \tfrac{1}{\sigma_k^2} \Big( \|g\|^2 - \tfrac{w_k^Tg}{\sigma_k^2} g^THw_k \Big)
    \le
    2\varepsilon \kappa\lambda_{max}
    <
    1.
\end{align}
Consequently, we get
\begin{align}%\label{eq:}
    \tfrac{|a_{k+1}|}{|a_k|}
    =
    \varepsilon_a \tfrac{|w_k^Tg|}{|a_k\sigma_k|} - (1-\varepsilon_a)
    \le
    2\varepsilon\varepsilon_a \kappa \lambda_{max} - (1-\varepsilon_a) < \varepsilon_a \le 1.
\end{align}

(3) Setting $r := \max(|1-\varepsilon_a|, 2\varepsilon\varepsilon_a \kappa \lambda_{max} - (1-\varepsilon_a))$, we finish the proof.
\end{proof}

To simplify our proofs for Theorem~\ref{th:convergence_BNGD}, we give two lemmas which are obvious but useful.
\begin{lemma}
\label{lemma:series2zero}
If positive series $f_k, h_k$ satisfy $f_{k+1} \le r f_k + h_k, r\in(0,1)$ and $\lim\limits_{k\to \infty}h_k = 0$, then $\lim\limits_{k\to \infty}f_k = 0$.
\end{lemma}
\begin{proof}
It is obvious, because the series $b_k$ defined by
$b_{k+1} = r b_k + h_k, b_0 >0$, tends to zeros.
\end{proof}

\begin{lemma}[Separation property]
\label{lemma:separation}
For $\delta_0$ small enough, the set $S := \{ w | y^2 q < \delta_0,\|w\|\ge 1\}$ is composed by two separated parts: $S_1$ and $S_2$, $dist(S_1,S_2) >0$, where in the set $S_1$ one has $y^2 < \delta_1, q > \delta_2$, and in $S_2$ one has $q < \delta_2, y^2 > \delta_1$  for some $\delta_1 >0, \delta_2 >0$. Here $y := w^Tg, q := u^THu - \tfrac{(w^THu)^2}{w^THw} = u^THu - \tfrac{y^2}{w^THw}$.
\end{lemma}
\begin{proof}
The proof is based on $H$ being positive. The geometric meaning is  illustrated in Figure~\ref{fig:separation}.
\begin{figure}[thb!]
  \center
  \includegraphics[width=7cm]{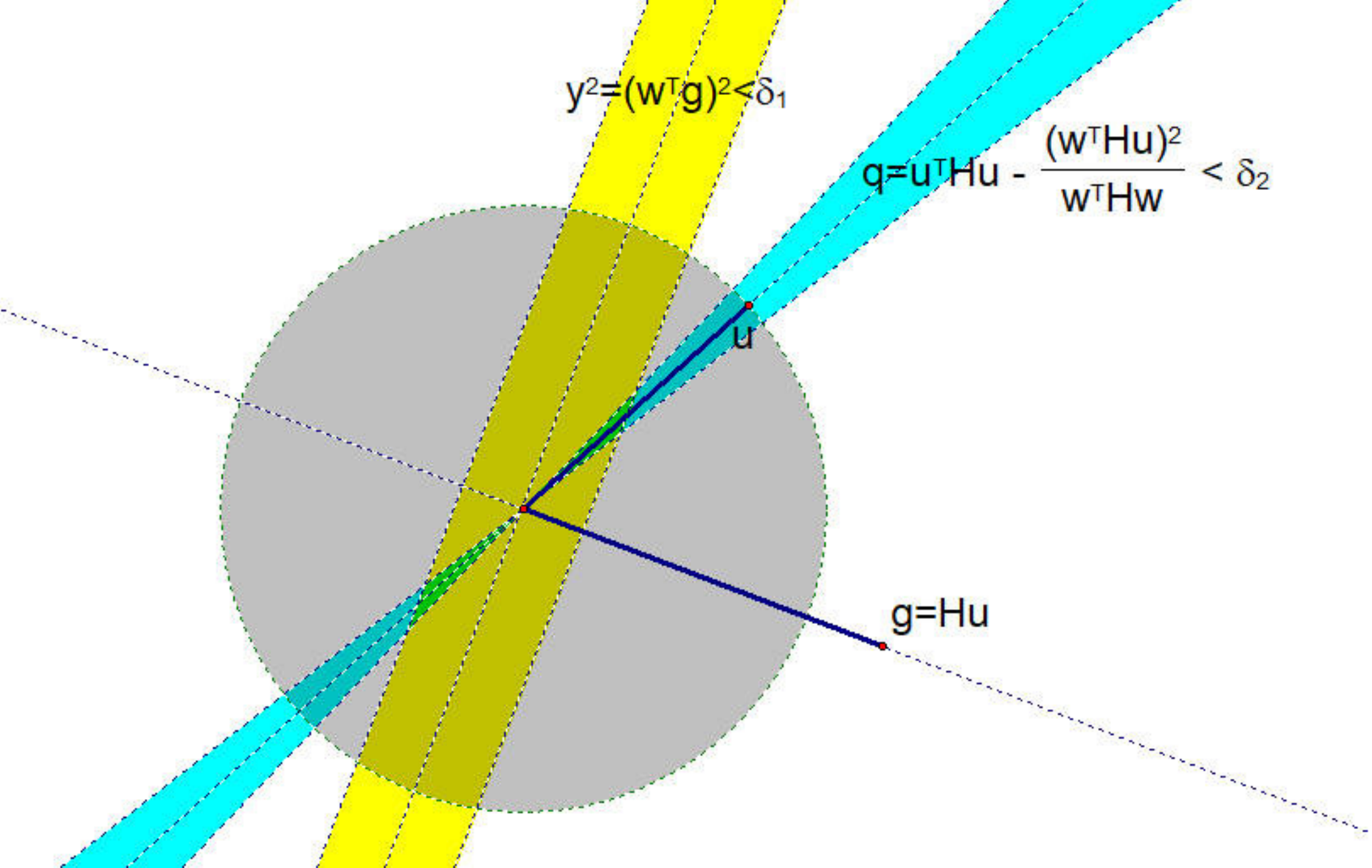}\\
  \caption{The geometric meaning of the separation property }\label{fig:separation}
\end{figure}
\end{proof}

\begin{corollary}
\label{cor:summable_either}
If $\lim\limits_{k\to\infty} \|w_{k+1}-w_k\|=0$, and $\lim\limits_{k\to\infty} (w_k^Tg)^2 q_k =0$, then either $ \lim\limits_{k\to\infty} (w_k^Tg)^2 =0 $
or $\lim\limits_{k\to\infty} q_k =0$.
\end{corollary}
\begin{proof}

Denote $y_k:= w_k^Tg$.
According to the separation property (Lemma~\ref{lemma:separation}), we can chose a $\delta_0>0$ small enough such that the separated parts of the set $S := \{ w |y^2 q < \delta_0,\|w\|\ge 1\}$,  $S_1$ and $S_2$,  have  $dist(S_1,S_2) > 0$.

Because $y_k^2 q_k$ tends to zero, we have $w_k$ belongs to $S$ for $k$ large enough, for instance $k>k_1$.
On the other hand, because $\|w_{k+1}-w_k\|$ tends to zero, we have  $\|w_{k+1}-w_k\| < dist(S_1,S_2)$ for $k$ large enough, for instance $k>k_2$.
Then consider $k > k_3:=\max(k_1,k_2)$, we have
all $w_k$ belongs to the same part $S_1$ or $S_2$.

If $w_k \in S_1$, ($q_k >\delta_2$), for all $k>k_3$, then we have
$\lim\limits_{k\to\infty} (w_k^Tg)^2=0$.

On the other hand, if $w_k \in S_2$, ($y_k^2 > \delta_1$), for all $k>k_3$, then we have $\lim\limits_{k\to\infty} q_k=0$.

\end{proof}

\begin{theorem}
\label{th:convergence_BNGD_general}
Let $\varepsilon_a \in (0,1]$ and $\varepsilon>0$. The sequence $(a_k,w_k)$ converges for any initial value $(a_0,w_0)$.
\end{theorem}

\begin{proof}
We will prove $\|w_k\|$ converges, and then prove $(a_k,w_k)$ converges as well.

(1) We prove that $\|w_k\|$ is bounded and hence converges.

In fact, according to the Lemma~\ref{lemma:convergence_BNGD_small2},
once $\|w_k\|^2 \ge \varepsilon/\varepsilon_0$ for some $k$, the rest of the iteration will converge, hence $\|w_k\|$ is bounded.

(2) We prove $\lim\limits_{k\to\infty} \|w_{k+1}-w_k\|=0$, and $\lim\limits_{k\to\infty} (w_k^Tg)^2 q_k = 0$.

The convergence of $\|w_k\|$ implies $\sum_{k} a_k^2 q_k$ is summable. As a consequence,
\begin{align}\label{eq:lim_ap}
    \lim_{k \to \infty}a_k^2 p_k=0, \lim_{k \to \infty}a_k e_k=0,
\end{align}
and $\lim\limits_{k\to\infty} \|w_{k+1}-w_k\|$ = 0. In fact, we have
\begin{align}%\label{eq:}
    \|w_{k+1}-w_k\|^2
    =
    \varepsilon^2 \tfrac{a_k^2}{\sigma^2} \|He_k\|^2
    \le
    \tfrac{\lambda_{max}\varepsilon^2}{\lambda_{min}^2} a_k^2 q_k
    \to 0.
\end{align}

Consider the iteration of series $|a_k - w_k^Tg/\sigma_k|$,
\begin{align}%\label{eq:}
    \Big|a_{k+1} - \tfrac{w_{k+1}^Tg}{\sigma_{k+1}}\Big|
    &\le
    \Big|a_{k+1} - \tfrac{w_{k+1}^Tg}{\sigma_k}\Big|
    +
    \Big|\tfrac{w_{k+1}^Tg}{\sigma_k} - \tfrac{w_{k+1}^Tg}{\sigma_{k+1}}\Big| \nonumber\\
    &\le
    (1-\varepsilon_a)\Big|a_k - \tfrac{w_k^Tg}{\sigma_k}\Big|
    +
    \varepsilon \tfrac{|a_k g^THe_k|}{\sigma_k^2}
    +
    \tfrac{|w_{k+1}^Tg|}{(\sigma_k\sigma_{k+1})}
    | \sigma_{k+1}-\sigma_{k}| \nonumber\\
    &\le
    (1-\varepsilon_a)\Big|a_k - \tfrac{w_k^Tg}{\sigma_k}\Big|
    +
    \varepsilon \tfrac{\|g\|_H \|a_k e_k\|_H}{\sigma_k^2}
    +
    \tfrac{|w_{k+1}^Tg|}{(\sigma_k\sigma_{k+1})}
    \varepsilon \tfrac{\lambda_{max}}{\sigma_k} \|a_k e_k\|_H \nonumber\\
    &\le
    (1-\varepsilon_a)\Big|a_k - \tfrac{w_k^Tg}{\sigma_k}\Big|
    +
    2 C \|a_ke_k\|_H \label{eq:ineq_ak}.
\end{align}
The constant $C$ in Eq.~(\ref{eq:ineq_ak}) can be chosen as
$C = \tfrac{\varepsilon \lambda_{max}\|u\|_H}{\lambda_{min}\|w_0\|^2}$.
Since $\|a_ke_k\|_H$ tends to zero, we can use Lemma~\ref{lemma:series2zero} to get $ \lim\limits_{k \to \infty}|a_k - w_k^Tg/\sigma_k| = 0$.
Combine the equation (\ref{eq:lim_ap}), then we have
$\lim\limits_{k \to \infty}(w_k^Tg)^2 p_k=0$.

(3) According to the Corollary~\ref{cor:summable_either}, we have either $\lim\limits_{k\to\infty} y_k^2 =0$,
or $\lim\limits_{k\to\infty} q_k =0$.
In the former case, the iteration of $(a_k, w_k)$ converges to a saddle point. However, in the latter case, $(a_k, w_k)$ converges to a global minimizer.
In both cases we have $(a_k,w_k)$ converges.

\end{proof}

To finish the proof of Theorem~\ref{th:convergence_BNGD}, we have to demonstrate the special case of $\varepsilon_a=1$ where the set of initial values such that BN iteration converges to saddle points is of Lebeguse measure zero. We leave this demonstration in next section where we consider the case of $\varepsilon_a\ge 1$.

\subsection{Impossibility of converging to strict saddle points}

In this section, we will prove the set of initial values such that BN iteration converges to saddle points is of Lebesgue measure zero, as long as $\varepsilon_a \ge 1$. The tools in our proof is similar to the analysis of gradient descent on non-convex objectives~\citep{Lee2016Gradient, Panageas2017Gradient}. In addition, we used the real analytic property of the BN loss function (\ref{eq:BN_loss}).

For brevity, here we denote $x:=(a,w)$ and let $\varepsilon_a = \varepsilon $, then the BN iteration can be rewritten as
$$x_{n+1} = T(x_n) := x_n - \varepsilon \nabla J(x_n).$$

\begin{lemma}
\label{lemma:measurable_T}
If $A \subset T(\mathbb{R}^d/\{0\})$ is a measure zero set, then the preimage $T^{-1}(A)$ is of measure zero as well.
\end{lemma}
\begin{proof}

Since $T$ is smooth enough, according to Theorem 3 of~\citet{Ponomarev1987Submersions}, we only need to prove the Jacobian of $T(x)$ is nonzero for almost all $x\in\mathbb{R}^d$. In other words, the set $\{x: \det(I-\varepsilon\nabla^2J(x))=0\}$ is of measure zero. This is true because the function $\det(I-\varepsilon\nabla^2J(x))$ is a real analytic function of $x\in \mathbb{R}^d/\{0\}$. (Details of properties of real analytic functions can be found in \citet{Krantz2002primer}).

\end{proof}

\begin{lemma}
\label{lemma:strict_saddle_local}
Let $f: X \to \mathbb{R}$ be twice continuously differentiable in an open set $X \subset \mathbb{R}^d$ and $x^* \in X$ be a stationary point of $f$.
If $\varepsilon>0$, $\det(I-\varepsilon\nabla^2f(x^*)) \neq 0$ and the matrix $\nabla^2f(x^*)$ has at least a negative eigenvalue,
then there exist a neighborhood $U$ of $x^*$ such that the following set $B$ has measure zero,
\begin{align}%\label{eq:}
     B:= \{x_0 \in U : x_{n+1}= x_n - \varepsilon \nabla f(x_n) \in U, \forall n \ge 0\}.
\end{align}
\end{lemma}
\begin{proof}
The detailed proof is similar to~\citet{Lee2016Gradient, Panageas2017Gradient}.

Define the transform function as $F(x) := x-\varepsilon\nabla f(x)$.
Since $\det(I-\varepsilon\nabla^2f(x^*))\neq0$, according to the inverse function theorem, there exist a neighborhood $U$ of $x^*$ such that $T$ has differentiable inverse. Hence $T$ is a local $C^1$ diffeomorphism, which allow us to use the central-stable manifold theorem~\citep{Shub2013Global}. The negative eigenvalues of $\nabla^2f(x^*)$ indicates $\lambda_{max}(I-\varepsilon\nabla^2f(x^*)) > 1$ and the dimension of the unstable manifold is at least one, which implies the set $B$ is on a lower dimension manifold hence $B$ is of measure zero.

\end{proof}

\begin{lemma}
\label{lemma:strict_saddle}
If $\varepsilon_a = \varepsilon \ge 1$, then the set of initial values such that BN iteration converges to saddle points is of Lebeguse measure zero.
\end{lemma}
\begin{proof}

We will prove this argument using Lemma~\ref{lemma:measurable_T} and Lemma~\ref{lemma:strict_saddle_local}. Denote the saddle points set as $W:=\{(a^*, w^*) : a^*=0, w^{*T}g=0\}$. The basic point is that the saddle point $x^* :=(a^*,w^*)$ of the BN loss function (\ref{eq:BN_loss}) has eigenvalues $\Big\{\frac12(1\pm \sqrt{1+ 4 \tfrac{\|g\|^2}{w^{*T}Hw^*}}),0,...,0\Big\}$ of the Hessian matrix.

(1) For each saddle point $x^* :=(a^*,w^*)$ of BN loss function, $\varepsilon\ge1$ is enough to allow us to use Lemma~\ref{lemma:strict_saddle_local}. Hence there exist a neighborhood $U_{x^*}$ of $x^*$ such that the following set $B_{x^*}$ is of measure zero,
\begin{align}%\label{eq:}
    B_{x^*} := \{x_0 \in U_{x^*} : x_n \in U_{x^*}, \forall n \ge 1\}.
\end{align}

(2) The neighborhoods $U_{x^*}$ of all $x^*\in W$ forms a cover of $W$, hence, according to Lindel\"{o}f's open cover lemma, there are countable neighborhoods $\{U_i : i=1,2,...\}$ cover $W$, i.e. $U := \cup_{i} U_i \supseteq W$. As a consequence, the following set $A_0$ is of measure zero,
\begin{align}%\label{eq:}
    A_0 := \cup_i B_i
    = \cup_i
    \{x_0 \in U_i : x_n \in U_i, \forall n \ge 1\}.
\end{align}

(3) Define $A_{m+1} := T^{-1} (A_m) = \{x \in \mathbb{R}^d: T(x) \in A_m\}, m\ge 0$. According to Lemma~\ref{lemma:measurable_T}, we have all $A_m$ and $\cup_{m} A_m$ are of measure zero.

(4) Since each initial value $x_0$ such that the iteration converges to a saddle point must be contained in some set $A_m$, we finish the proof.

\end{proof}

Combine the results of Lemma~\ref{lemma:strict_saddle}, scaling property \ref{prop:scaling} and the convergence theorem \ref{th:convergence_BNGD_general}, we have the following theorem directly.

\begin{theorem}
\label{th:BN_asto_minimizer}
If $\varepsilon_a=1,\varepsilon \ge 0$, then the BN iteration  (\ref{eq:BNGD_a})-(\ref{eq:BNGD_w}) converges to global minimizers for almost all initial values.
\end{theorem}

\subsection{Convergence rate}

In section \ref{sec:proof_3p3}, we encountered the following estimate for $e_k=u-\tfrac{w_k^Tg}{\sigma_k^2} w_k$
\begin{align}
    \|e_{k+1}\|_H \le \rho(I - \hat\varepsilon_k H) \|e_k\|_H.
\end{align}
We can improve the convergence rate of the above if $H^*$ has better spectral property. This is the content of Theorem~\ref{th:convergence_rate_BNGD} and the following lemma proves this.
\begin{lemma}
\label{lemma:converge_rho_star}
The following inequality holds,
\begin{align}%\label{eq:}
    (1-\delta_k) \|e_{k+1}\|_H
    \le
    \Big(
        \rho^*(I - \hat\varepsilon_k H^*)
        + \delta_k
    \Big)
    \|e_k\|_H,
\end{align}
where $\delta_k := \tfrac{\lambda_{max}\varepsilon |a_k|}{\sigma_k^2}\|e_k\|_H$.
\end{lemma}

\begin{proof}
The case of  $w_k^Tg = 0$ is trivial, hence we assume $w_k^Tg \neq 0$ in the following proof. Rewrite the iteration on $w_k$ as the following equality,
\begin{align}\label{eq:iteration_Hstar}
    u - \tfrac{w_k^Tg}{\sigma_k^2} w_{k+1}
    =
    (I - \hat\varepsilon_k H) e_k
    =
    (I - \hat\varepsilon_k H^*) e_k
    -\hat\varepsilon_k
    \Big(1-\tfrac{(w_k^T g)^2}{u^THu \sigma_k^2} \Big) Hu.
\end{align}
Then we will use the properties of $H^*$-seminorm to prove our argument.

(1) Estimate the $H^*$-seminorm on the right hand of Eq.~(\ref{eq:iteration_Hstar}).
\begin{align}%\label{eq:}
    \|\text{right}\|_{H^*}
    &\le
    \|(I - \hat\varepsilon_k H^*) e_k\|_{H^*}
    +|\hat\varepsilon_k|
    \Big(1-\tfrac{(w_k^T g)^2}{u^THu \sigma_k^2} \Big)
    \|Hu\|_{H^*} \\
    &\le
    \rho^*(I - \hat\varepsilon_k H^*) \|e_k\|_{H^*}
    +
    \tfrac{\lambda_{max}|\hat\varepsilon_k|}{\sqrt{u^THu}} \|e_k\|_{H}^2 \\
    &=
    \rho^*(I - \hat\varepsilon_k H^*) \tfrac{|w_k^Tg|}{\sqrt{u^THu} \sigma_k} \|e_k\|_H
    +
    \tfrac{\lambda_{max}\varepsilon |a_k w_k^Tg|}{\sqrt{u^THu} \sigma_k^3} \|e_k\|_{H}^2 \\
    &=
    \tfrac{|w_k^Tg|}{\sqrt{u^THu} \sigma_k}
    \Big(
        \rho^*(I - \hat\varepsilon_k H^*)
        + \delta_k
    \Big)
    \|e_k\|_H.
\end{align}

(2) Estimate the $H^*$-seminorm on the left hand of equation (\ref{eq:iteration_Hstar}).
Using the $H$-norm on the iteration of $w_k$, we have
\begin{align}%\label{eq:}
    \sigma_{k+1}
    &=
    \|w_k + \varepsilon\tfrac{a_k}{\sigma_k} H e_k\|_H
    \ge
    \sigma_k - \varepsilon\tfrac{\lambda_{max}|a_k|}{\sigma_k} \|e_k\|_H.
\end{align}
Consequently, we have
\begin{align}%\label{eq:}
    \|\text{left}\|_{H^*}
    &=
    \tfrac{|w_k^Tg|}{\sqrt{u^THu} \sigma_k} \tfrac{\sigma_{k+1}}{\sigma_k} \|e_{k+1}\|_H
    \ge
    \tfrac{|w_k^Tg|}{\sqrt{u^THu} \sigma_k}
    (1-\delta_k) \|e_{k+1}\|_H.
\end{align}

(3) Combining (1) and (2), we finish the proof.
\end{proof}

Then we give the proof of Theorem \ref{th:convergence_rate_BNGD}.

\begin{proof}[Proof of Theorem \ref{th:convergence_rate_BNGD}]
    Firstly, the Lemma \ref{lemma:converge_rho_star} implies the second part of Theorem \ref{th:convergence_rate_BNGD} which is the special case of $\delta_k<\delta<1$.

    Secondly, if $\hat\varepsilon < \varepsilon_{max}^*$, then $\rho^*(I-\hat\varepsilon H^*)<1$. Since $(a_k, w_k)$ converges to a minimizer, $\delta_k$ must converge to zero and the coefficient
    $\tfrac{\rho^*(I - \hat\varepsilon_k H^*)
        + \delta_k}{(1-\delta_k)}$
    must less than a number $\hat\rho \in (0,1)$ when $k$ is large enough which results in the linear convergence of $\|e_k\|_H$.
\end{proof}

Now, we turn to the convergence of the loss function which can be rewritten as $J_k = \tfrac12 \|\tilde e_k\|^2_H$ with $\tilde e_k = u-\tfrac{a_k}{\sigma_k}w_k$. There is an useful equality between $\|\tilde e_k\|^2_H$ and $\|e_k\|^2_H$:
\begin{align}\label{eq:loss_and_ek}
    \|\tilde e_k\|^2_H = \|e_k\|^2_H + \Big(a_k - \tfrac{w_k^Tg}{\sigma_k}\Big)^2.
\end{align}
Recalling the inequality (\ref{eq:ineq_ak}) and the boundedness of $a_k$, we have a constant $C_0$ such that
\begin{align}\label{eq:ineq_ak_V2}
    \Big|a_{k+1} - \tfrac{w_{k+1}^Tg}{\sigma_{k+1}}\Big|
    \le
    |1-\varepsilon_a|\Big|a_k - \tfrac{w_k^Tg}{\sigma_k}\Big|
    +
    C_0 \|e_k\|_H,
\end{align}
which indicates that we can use the convergence of $e_k$ to estimate the convergence of the loss value $J_k$. In fact we have the following lemma.

\begin{lemma}
\label{lemma:convergence_loss}
If $\|e_k\|_H \le C\rho^k$ for some constant $C$ and $\rho\in(0,1)$, $\varepsilon_a \in (0,1]$, then we have
\begin{align}%\label{eq:}
    \|\tilde e_k\|^2_H
    \le
    C^2 \rho^{2k} + \Big( C_1 (1-\varepsilon_a)^k
    + C_2 k \gamma^k \Big)^2,
\end{align}
where $\gamma=\max(\rho,1-\varepsilon_a)$, $C_1 = |a_0-w_0^Tg/\sigma_0|$ and $C_2 = CC_0$.
\end{lemma}

\begin{proof}
According to the inequality (\ref{eq:ineq_ak_V2}), we have
\begin{align}%\label{eq:}
    \Big|a_k - \tfrac{w_k^Tg}{\sigma_k}\Big|
    \le
    C_1 (1-\varepsilon_a)^k +
    C_2 \sum_{i=0}^{k-1} (1-\varepsilon_a)^i \rho^{k-i}
    \le
    C_1 (1-\varepsilon_a)^k +
    C_2 k \gamma^k.
\end{align}
Put it in the Eq.~(\ref{eq:loss_and_ek}), then we finish the proof.

\end{proof}

\subsection{Estimating the effective step size}
\label{sec:estimate_hat_eps}

{%\color{red}
Firstly, we consider the limit of effective step size $\hat\varepsilon$.
When the iteration converges to a minimizer $(a^*,w^*)$, the value of $\hat\varepsilon$ is $\hat\varepsilon=\tfrac{\varepsilon}{\|w^*\|^2}$. Without loss generality, we assume that $w_k$ always has different direction with $u$ during the whole course of the iterations. In fact, if $w_k$ has the same direction with $u$ for some $k$, then the iteration of $w_k$ is trivial, i.e. $w_k = w_{k+1}=w_{k+2}=...$, and the effective step size can be any positive number. However, this case is rare. More precisely, we have the following lemma:
\begin{lemma}
    \label{lemma:BN_asto_linear}
    The set of initial values $(a_0, w_0)$ such that  $(a_k, w_k)$ converges to a minimizer $(a^*, w^*)$ with effective learning rate $\hat\varepsilon:=\lim\limits_{k \to \infty} \hat\varepsilon_k > \varepsilon_{max}^*$ and $\det(I-\hat\varepsilon H^*) \neq 0$ is of measure zero.
\end{lemma}
\begin{proof}
The proof is similar to the proof of Lemma \ref{th:BN_asto_minimizer}. The key point is that the matrix $I-\hat\varepsilon H^*$ at this minimizer is non-degenerate and has an eigenvalue with its absolute value large than 1, hence there is a local unstable manifold with dimension greater than one.
\end{proof}
}

Now we consider the effective learning rate $\hat\varepsilon_k$ and give the proof of Proposition \ref{prop:estimate_stepsize}.

According to Lemma~\ref{lemma:convergence_BNGD_small}, the effective step size $\hat\varepsilon_k$ has same order with $\tfrac{\varepsilon}{\|w_k\|^2}$ provided $a_0w_0^Tg>0, \varepsilon/||w_0|| < \varepsilon_0$. In fact, we have
\begin{align}%\label{eq:}
    \tfrac{C_1 \varepsilon}{\|w_k\|^2} :=
    \tfrac{a_0 w_0^Tg}{\sigma_0} \tfrac{\varepsilon}{\lambda_{max}\|w_k\|^2}
    \le
    \hat\varepsilon_k
    \le
    \sqrt{u^THu} \tfrac{C_a \varepsilon}{\lambda_{min}\|w_k\|^2}
    =: \tfrac{C_2 \varepsilon}{\|w_k\|^2} .
\end{align}
Hence, to prove the Proposition~\ref{prop:estimate_stepsize}, we only need to estimate the norm of $w_k$.
%Note that we assume $\varepsilon_a \in (0,1], a_0w_0^Tg>0, ||g||^2 \ge \tfrac{w_0^Tg}{\sigma_0^2} g^THw_0$ and $(a_0,w_0)$ is not a minimizer.

\begin{proof}[Proof of Proposition~\ref{prop:estimate_stepsize}]
According to the BNGD iteration, we have (see the proof of Lemma~\ref{lemma:convergence_basic})
\begin{align}%\label{eq:}
    \|w_{k+1}\|^2
    %&= \|w_k\|^2 + \varepsilon^2 \tfrac{a_k^2}{\sigma_k^2} \|He_k\|^2\\
    %&=     \|w_0\|^2 + \varepsilon^2 \sum_{i=0}^k \tfrac{a_i^2}{\sigma_i^2}  \|He_i\|^2\\
    & \le
    \|w_0\|^2 + \varepsilon^2 \lambda_{max}\sum_{i=0}^k \tfrac{a_i^2}{\sigma_i^2} q_i.
\end{align}

(1) When $\tfrac{\varepsilon}{\|w_0\|^2}<\varepsilon_0$ ($\varepsilon_0$ is defined in Lemma~\ref{lemma:convergence_BNGD_small}), the sequence $q_k$ satisfies $q_{k+1} \le (1- \hat\varepsilon_k \lambda_{min}) q_k$. Hence the norm of $w_k$ is bounded by
\begin{align}%\label{eq:}
    \|w_{k}\|^2
    \le
    \|w_0\|^2 + \varepsilon \kappa C_a \tfrac{\sigma_0}{w_0^Tg}\sum_{i=0}^\infty ( q_{i}-q_{i+1})
    \le
    \|w_0\|^2 + C \varepsilon,
\end{align}
for some constant $C$. As a consequence,
\begin{align}%\label{eq:}
    \tilde C_1 \varepsilon
    :=
    \tfrac{C_1 \varepsilon}{\|w_0\|^2(1 + C\varepsilon_0)}
    \le
    \hat\varepsilon_k
    \le
    \tfrac{C_2 \varepsilon}{\|w_0\|^2}
    =: \tilde C_2 \varepsilon.
\end{align}

(2) When $\varepsilon$ is large enough, the increment of the norm $\|w_k\|$ at the first step is large as well. In fact, we have
\begin{align}\label{eq:normw1w0}
    \|w_1\|^2 - \|w_0\|^2
    =
    \varepsilon^2 \tfrac{a_0^2}{\sigma_0^2} \|He_0\|^2
    = C_3 \varepsilon^2.
\end{align}
Since $||g||^2 \ge \tfrac{w_0^Tg}{\sigma_0^2} g^THw_0$, we have $a_1 w_1^Tg >a_1 w_0^Tg > 0$. Choose $\varepsilon$ to be larger than some value $\varepsilon_1$ such that $\tfrac{\varepsilon}{\|w_1\|^2}<\varepsilon_0$, then we can use the argument in (1) on $(a_1,w_1)$. More precisely, there are two constants, $C_1, C_2$, such that
\begin{align}%\label{eq:}
    \tfrac{C_1 \varepsilon}{\|w_1\|^2}
    \le
    \hat\varepsilon_k
    \le
    \tfrac{C_2 \varepsilon}{\|w_1\|^2}.
\end{align}
Plugging the equation (\ref{eq:normw1w0}) into it, we have
\begin{align}%\label{eq:}
    \tfrac{C_1 \varepsilon_1^2}{\|w_0\|^2 + C_3 \varepsilon_1^2}
    \le
    \tfrac{C_1 \varepsilon^2}{\|w_0\|^2 + C_3 \varepsilon^2}
    \le
    \hat\varepsilon_k \varepsilon
    \le
    \tfrac{C_2 \varepsilon^2}{\|w_0\|^2 + C_3 \varepsilon^2}
    \le
    \tfrac{C_2}{C_3}.
\end{align}

\end{proof}

% Note that the condition $||g||^2 \ge \tfrac{w_0^Tg}{\sigma_0^2} g^THw_0$ is introduced only for simplifying the proof.

{%\color{red}

\subsection{Quantification of the insensitive interval}
\label{sec:omega}

In this section, we estimate the magnitude of insensitive interval of step size.

% \subsubsection{Approximate $\|w_k\|^2$ by ODE}

The BNGD iteration with configuration $\varepsilon_a=1, a_0=\tfrac{w_0^Tg}{\sigma_0}, \|w_0\|=\|u\|=1$ implies the following equality of $\|w_k\|^2$,
\begin{align}\label{eq:normw_iteration}
    \|w_{k+1}\|^2 &=\|w_k\|^2  +
\tfrac{\varepsilon^2}{\|w_k\|^2} \tfrac{a_k^2\|w_k\|^2}{\sigma_k^2}
\big\| e_k\big\|_{H^2}^2 \nonumber\\
&=:
\|w_k\|^2  +
\tfrac{\varepsilon^2}{\|w_k\|^2} \beta_k ,
\end{align}
where $\beta_k$ is defined as $\beta_k:=  \tfrac{a_k^2\|w_k\|^2}{\sigma_k^2}
\big\| e_k\big\|_{H^2}^2$. The linear convergence results allow us to assume that $\beta_k$ converges linearly to zero, ~i.e. $\beta_k = \beta_0 \rho^{k}, k\ge 0$ where $\rho\in(0,1)$ depends on $\varepsilon$ and is self-consistently determined by the limiting effective step size, ~i.e. $\rho = \rho(I-\tfrac{\varepsilon}{\|w_\infty\|^2}H)$ is the spectral radius of $I-\tfrac{\varepsilon}{\|w_\infty\|^2}H$. Observed that the iteration in Eq.~\eqref{eq:normw_iteration} can be regarded as a numerical scheme for solving the following ODE:

\begin{align}
    \xi(0)=\|w_1\|^2, \dot \xi(t)  = \frac{\varepsilon^2\beta_0 \rho^{2t}}{\xi(t)},
\end{align}
which has solution $\xi^2(t) = \xi^2(0) + \tfrac{\varepsilon^2 \beta_0}{|\ln \rho|}(1-\rho^{2t})$, the value of $\|w_k\|^2$ can be approximated by $\xi(k+1)$. Particularly, we have an approximation for $\|w_\infty\|^2$:
\begin{align}
    \|w_\infty\|^2 \approx \xi(\infty) = \sqrt{(1+\varepsilon^2\beta_0)^2 + \tfrac{\varepsilon^2 \beta_0}{|\ln \rho|}}.
\end{align}

To determine the value of $\rho$, we let $\rho$ and $\varepsilon$ satisfy the following relation:
\begin{align}
    \rho = \rho(I-\tfrac{\varepsilon}{\xi(\infty)}H):=\max_i\{|1-\tfrac{\varepsilon}{\xi(\infty)}\lambda_i(H)|\},
\end{align}
which closed the calculation of $\xi(\infty)$.

Next, we consider two limiting case: $\varepsilon\ll1$ and $\varepsilon\gg1$. In both case, the effective step size $\hat\varepsilon$ is small, and the value of $\rho$ is related to
$\rho = 1- \frac{\varepsilon \lambda_{min}}{\xi(\infty)}$.
Combine the definition of $\xi(\infty)$, then we have
\begin{align}
    \tfrac{\varepsilon^2 \lambda^2_{min}}{(1-\rho)^2} =
    \xi(\infty)^2
    =
    (1+\varepsilon^2\beta_0)^2 + \tfrac{\varepsilon^2 \beta_0}{|\ln \rho|}
    \approx
    (1+\varepsilon^2\beta_0)^2 + \tfrac{\varepsilon^2 \beta_0}{1-\rho},
\end{align}
where the estimate of $|\ln\rho| \approx 1-\rho$, is used since $\rho$ is closed to 1 for $\hat\varepsilon$ is small enough. Consequently, we have:
\begin{itemize}
    \item[(1)] When $\varepsilon \ll 1$, we have $\alpha^* \approx 1$, $\rho \approx 1 - \varepsilon\lambda_{min}$ and $\hat\varepsilon
    \approx \varepsilon$.
    \item[(2)] When $\varepsilon \gg 1$, we have
    \begin{align}
        \hat\varepsilon \approx \frac{1-\rho}{\lambda_{min}} \approx \frac{\sqrt{1 + 4 \varepsilon^2 \lambda_{min}^2}-1}{2 \varepsilon^2 \beta_0 \lambda_{min}}
    =
    \frac{1}{\beta_0}\frac{2\lambda_{min}}{\sqrt{1 + 4 \varepsilon^2 \lambda_{min}^2}+1}
    \sim \tfrac{1}{\beta_0\varepsilon}.
    \end{align}
\end{itemize}
Those results indicate the magnitude of insensitive interval of step size is proportion to the constant $\tfrac1\beta_0$.

Finally, we estimate the average of $\beta_0$ over $w_0$ and $u$ for given $H$.
The average value of $\beta_0$ from BNGD is defined as the following geometric average over $w_0$ and $u$, which we take to be independent and uniformly on the unit sphere $\mathbb{S}^{d-1}$,
\begin{align}\label{eq:definition_bar_beta}
    \bar\beta_H :=
    \mathbb{E}_{w_0,u}^G
    [ \beta_0]
    :=
    \exp\big(
        \mathbb{E}_{w_0,u}
        \ln\big[ \big(\tfrac{w_0^T Hu}{w_0^THw_0}\big)^2
        \big\| e_0\big\|_{H^2}^2 \big]
        \big).
    %\le
    %\mathbb{E}_{w_0,u} \big[ \big(\tfrac{w_0^T Hu}{w_0^THw_0}\big)^2 \big\| e_0\big\|_{H^2}^2 \big].
\end{align}

Correspondingly, the magnitude of insensitive interval of step size is defined as $\Omega$,
\begin{align}\label{eq:definition_Ogema}
    \Omega = \Omega_H :=
    \mathbb{E}_{w_0,u}^G [\tfrac{\lambda^2_{max}(H)}{4\beta_0}]
    = \tfrac{\lambda^2_{max}(H)}{4 \bar\beta_H}.
\end{align}

% It is worth noting that we prefer the geometric mean to the arithmetic mean, because the arithmetic mean is dominated by the extreme cases while the geometric mean is more stable in numerical tests in section \ref{sec:test_BN1GD}.
The numerical tests find that $\Omega_H$ highly depends on the dimension $d$ provided the eigenvalues of $H$ is sampled from typical distributions such as the uniform distribution on $[\lambda_{min}, \lambda_{max}]$ with $0<\lambda_{min}<\lambda_{max}$. In fact we have the following estimations for $\bar\beta_H$ which implies $\bar\beta_H\le O(1/d)$ and $\Omega_H \ge O(d)$.
\begin{lemma}\label{lemma:estimate_bar_beta_H}
    For positive definite matrix $H$ with minimal and maximal eigenvalues, $\lambda_{min}$ and $\lambda_{max}$ respectively, the $\bar\beta_H$ defined in (\ref{eq:definition_bar_beta}) satisfies,
    \begin{align}
        \bar\beta_H
        &\le
        \tfrac1d \tfrac{Tr[H^2]}{d} \tfrac{\lambda_{max} Tr[H]}{d} \tfrac{1}{\lambda^2_{min}} \exp\big(-\tfrac{2\ln \kappa}{\kappa-1} (\tfrac{Tr[H]}{d\lambda_{min}}-1)\big),% = O(1/d),\\
        %\Omega_H &= \tfrac{\lambda^2_{max}(H)}{4 \bar\beta_H} \ge O(d),
    \end{align}
    where $\kappa=\tfrac{\lambda_{max}}{\lambda_{min}}$ is the condition number of $H$.
\end{lemma}
\begin{proof}
    The definition of $\mathbb{E}^G$ allows us to estimate each term in $\beta$ separately.

(1). The inequality of arithmetic and geometric means implies $\mathbb{E}^G[(w_0^THu)^2] \le \mathbb{E}[(w_0^THu)^2]=\tfrac{Tr[H^2]}{d^2}$.

(2). Using the definition of $e_0 = u-\tfrac{w_0^THu}{w_0^THw_0}w_0$, we have
\begin{align*}
    \mathbb{E}^G \big[\| e_0\|_{H^2}^2 \big]
\le \lambda_{max} \mathbb{E} \big[\| e_0\|^2_{H} \big]
&= \lambda_{max} \mathbb{E} \big[\|u-\tfrac{w_0^THu}{w_0^THw_0}w_0\|^2_{H} \big]\\
&= \lambda_{max} \mathbb{E} \big[u^THu-\big(\tfrac{w_0^THu}{w_0^THw_0}\big)^2 \big]\\
&\le \lambda_{max} \mathbb{E} \big[u^THu\big]
= \tfrac{\lambda_{max} Tr[H]}{d}.
\end{align*}

(3). Since $w_0^THw_0 \in [\lambda_{min,\lambda_{max}}]$, using the fact that $\ln(1+x)\ge \tfrac{\ln \kappa}{\kappa-1}, \forall x \in [0,\kappa-1]$, we have
\begin{align}
    \mathbb{E}^G \big[w_0^THw_0 \big]
&=\exp\big(\mathbb{E}\ln (w_0^THw_0\big)\big) \nonumber\\
&\ge
\lambda_{min} \exp\big(\mathbb{E}\tfrac{\ln \kappa}{\kappa-1} (w_0^THw_0/\lambda_{min}-1)\big)\\
&=
\lambda_{min} \exp\big(\tfrac{\ln \kappa}{\kappa-1} (\tfrac{Tr[H]}{d\lambda_{min}}-1)\big). \nonumber
\end{align}

Combine the inequities above, then we finish the proof.
\end{proof}

If the eigenvalues of $H$ is sampled from a given distribution on $[\lambda_{min},\lambda_{max}]$, the values $\tfrac{Tr[H]}{d}, \tfrac{Tr[H^2]}{d}$ are related to the distribution and not sensitive to dimension $d$ (for $d$ large enough), then the estimate in Lemma \ref{lemma:estimate_bar_beta_H} indicates that $\bar\beta_H\le O(1/d)$ and $\Omega_H \ge O(d)$.
As an example, we consider the $H$ with eigenvalues forming an arithmetic sequence below.

\begin{corollary}\label{cor:H_lin_beta_Omega}
    If the eigenvalues of $H$ are $\lambda_i = \lambda_{min} + (i-1) \tfrac{\lambda_{max}-\lambda_{min}}{d-1}, d\ge 2$, then we have
\begin{align}
    \bar\beta_H \le \tfrac{(\kappa+1)^3}{\kappa^2}\tfrac{\lambda_{max}^2}{4d},~
    \Omega_H \ge \frac{\kappa^2}{(\kappa+1)^3} d.
\end{align}
\end{corollary}
\begin{proof}
    It is enough to show that $\tfrac{Tr[H]}{d}=\tfrac{(\kappa+1)\lambda_{min}}{2}, \tfrac{Tr[H^2]}{d^2} \le \tfrac{(\kappa+1)^2\lambda_{min}^2}{2d}$.
\end{proof}

The Corollary \ref{cor:H_lin_beta_Omega} indicates that larger dimensions lead to larger insensitive intervals of step size. It is interesting to note that although the lower bound of $\Omega_H$ is also related to the condition number $\kappa$, the numerical tests in section \ref{sec:test_BN1GD} find the width is not sensitive to $\kappa$. In fact, one could get better lower bounds for $\Omega_H$ by better estimates on $\mathbb{E}^G(w^THw)$. However, here we focus on the effect of dimension.

\subsubsection{Numerical tests}
\label{sec:test_BN1GD}

In this section, we give some numerical tests on the BNGD iteration with $\varepsilon_a=1, a_0=0$ and choices of the matrix $H$. The scaling property allows us to set $H$ diagonal and the initial value $w_0$ having the same norm with $u$, $\|w_0\| = \|u\| = 1$.

Firstly, we show the difference of geometric mean(G-mean) and arithmetic mean(A-mean) in quantifying the performance of BNGD.
Figure~\ref{fig:res500_random1_d100} gives an example of a 100-dimensional $H$ with condition number $\kappa = 853$. The GD and MBNGD iteration are executed $k=5000$ times where $u$ and $w_0$ are randomly chosen from the unit sphere. The values of effective step size, loss $\|e_k\|^2_H$ and error $\|e_k\|$ are plotted. Furthermore, the mean values over 500 random tests are given. The results show that the G-mean converges quickly when the number of tests increase, however the A-mean does not converge as quickly and A-mean is dominated by the largest sample values. Hence we use the geometric mean in later tests.

\begin{figure}[thb!]
    \center
    \includegraphics[width=12cm]{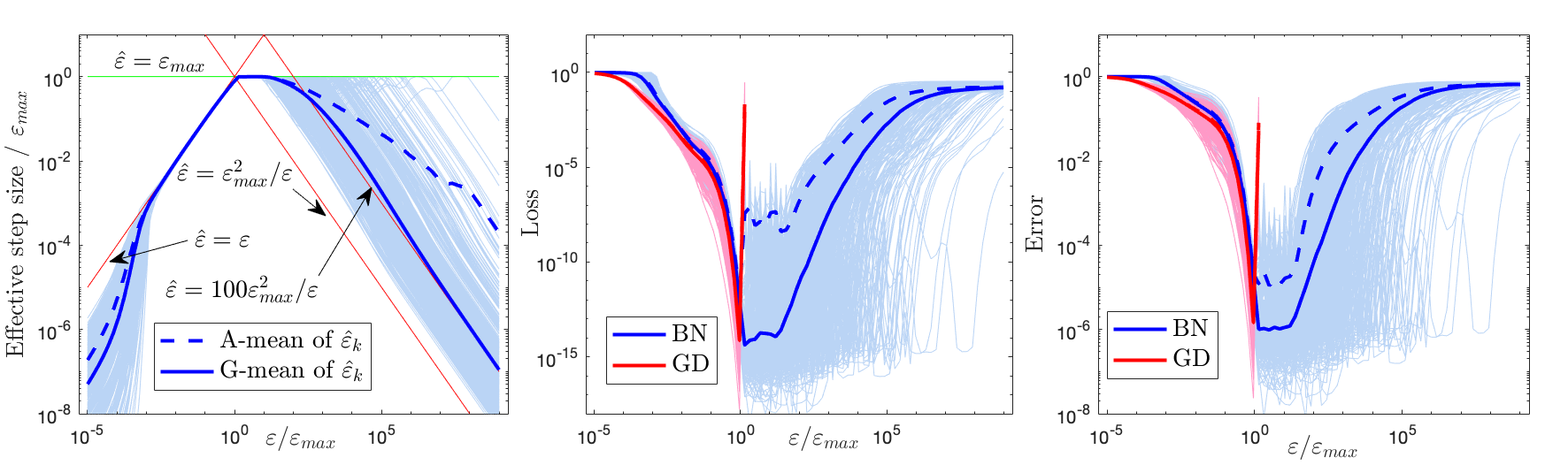}\\
    \caption{
        {%\color{red}
        Test BNGD on OLS model with step size $\varepsilon_a=1, a_0=0$.
        Parameters: $H$ is a diagonal matrix with condition number $\kappa=853$ (the first random test in Figure~\ref{fig:BN1GD_random_d100_res500}), $u$ and $w_0$ is randomly chosen uniformly from the unit sphere in $\mathbb{R}^{100}$. The BNGD iterations are executed for $k=5000$ steps. The bold curves are averaged over the 500 independent runs (the shadow curves).
        }
    }
    \label{fig:res500_random1_d100}
\end{figure}

Secondly, we test the effect of dimension $d$.

Figure~\ref{fig:BN1GD_10000_dvar_res500} gives three typical setting of $H$: (a) with arithmetic progression eigenvalues, (b) with geometric progression eigenvalues and (c) with only one large eigenvalue perturbed from identity matrix. In the first two cases, the effect of dimension is observed, the large dimensions lead to large magnitude $\Omega$ of optimal step size, and the magnitude is almost proportion to the dimension $d$ which confirm the analysis in Lemma~\ref{lemma:estimate_bar_beta_H} and Corollary~\ref{cor:H_lin_beta_Omega}. In the last case, the large dimensions lead to small $\Omega$ which is due to $Tr[H]/d$ and $Tr[H^2]/d$ are highly influenced by $d$. However, the condition number of $H^*$ be much smaller than $\kappa(H)$, in which case leads to marked acceleration over GD.

\begin{figure}[thb!]
    \center
    \includegraphics[width=12cm]{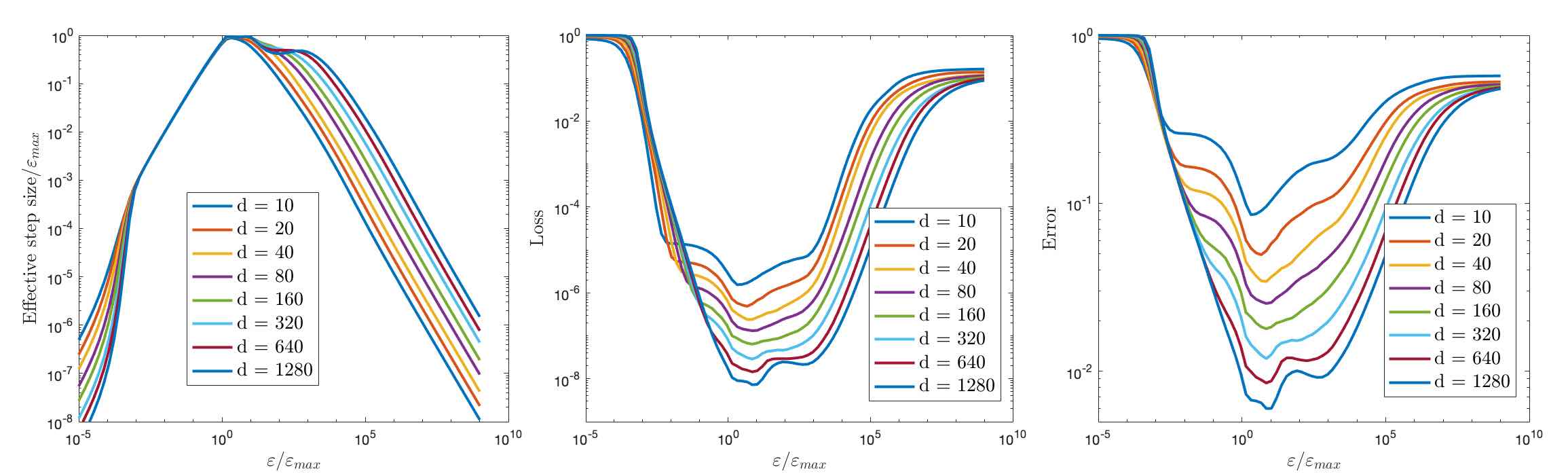}\\
    \includegraphics[width=12cm]{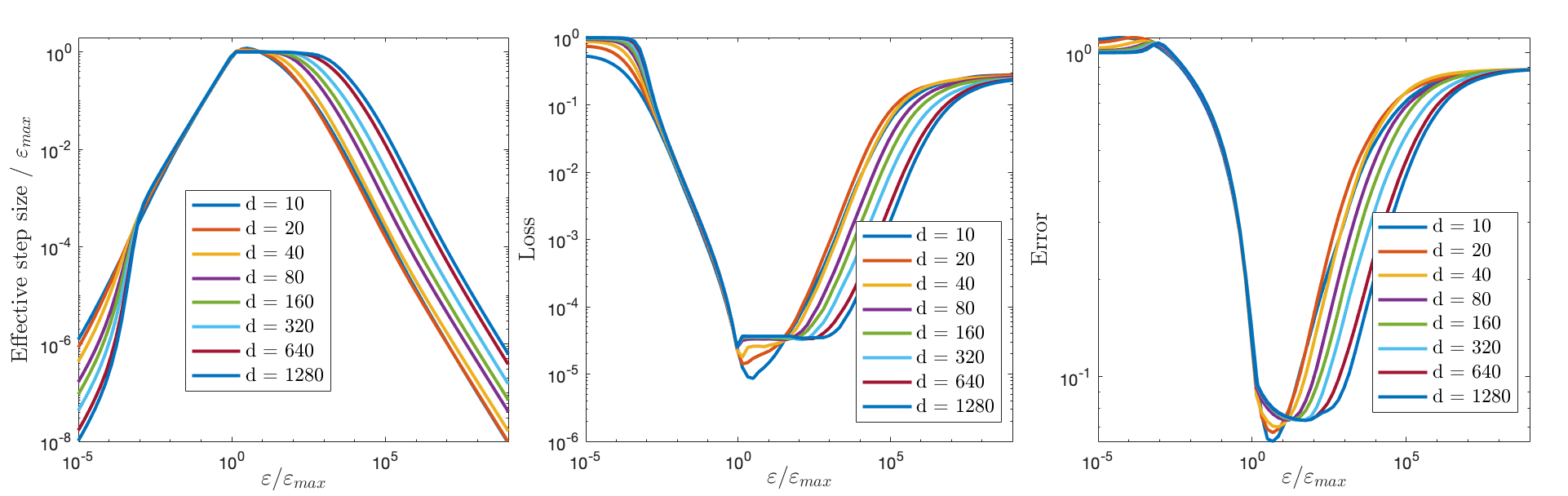}\\
    \includegraphics[width=12cm]{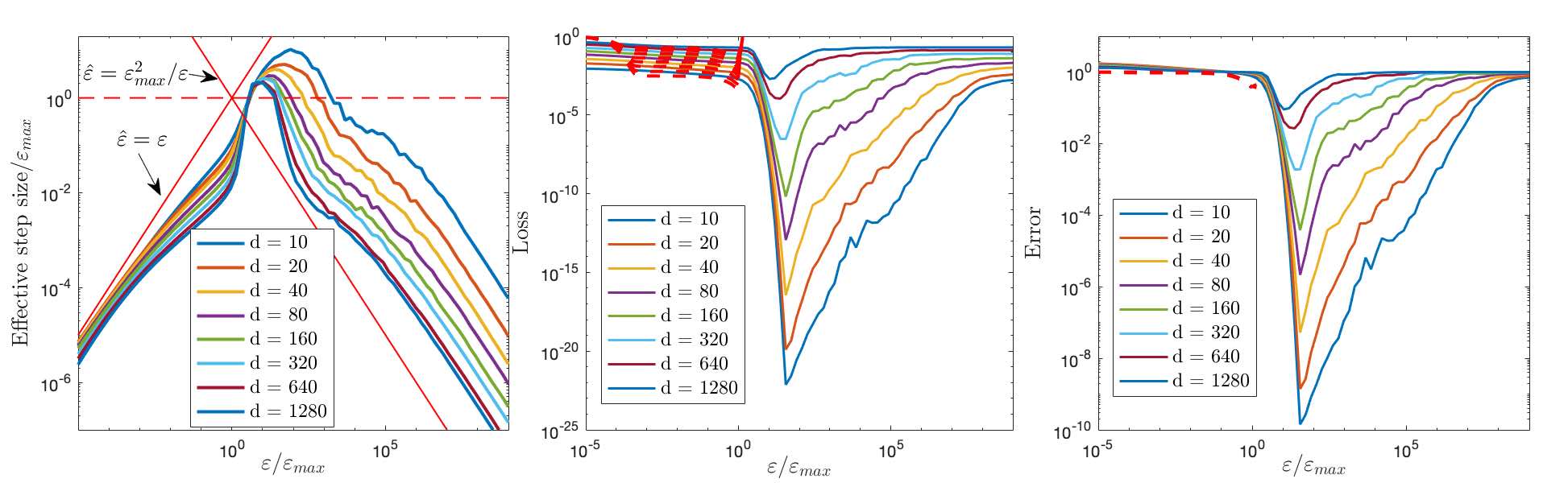}\\
    \caption{
        {%\color{red}
        Tests of BNGD on OLS model with step size $\varepsilon_a=1, a_0=0$.
        Parameters:
        (a, top) $H$ = diag(linspace(1,10000,d)),
        (b, middle) $H$ = diag(logspace(0,4,d)),
        (c, bottom) $H$ = diag([ones(1,d-1),10000]]).
        $u$ and $w_0$ is randomly chosen uniformly from the unit sphere in $\mathbb{R}^{d}$. The BNGD iterations are executed for $k=5000$ steps. The curves are averaged over the 500 independent runs.
        }
    }
    \label{fig:BN1GD_10000_dvar_res500}
\end{figure}

Finally, we test the effect of eigenvalue distributions. Figure~\ref{fig:BN1GD_random_d100_res500} gives examples of $H$ with different condition number but same dimension $d=100$. When the eigenvalues are arithmetic sequences, the width of optimal learning rate is almost same over different condition numbers while the loss and error still depend on the condition number. Randomly choosing eigenvalues also exhibits this phenomenon.

\begin{figure}[thb!]
    \center\includegraphics[width=12cm]{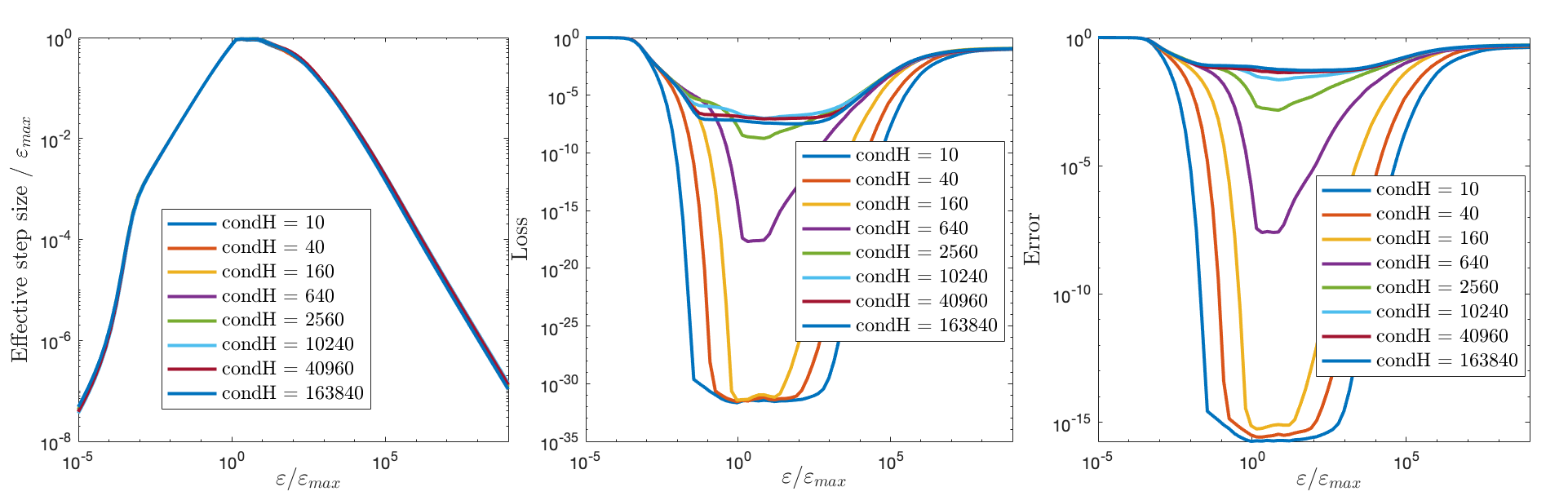}\\
    \center\includegraphics[width=12cm]{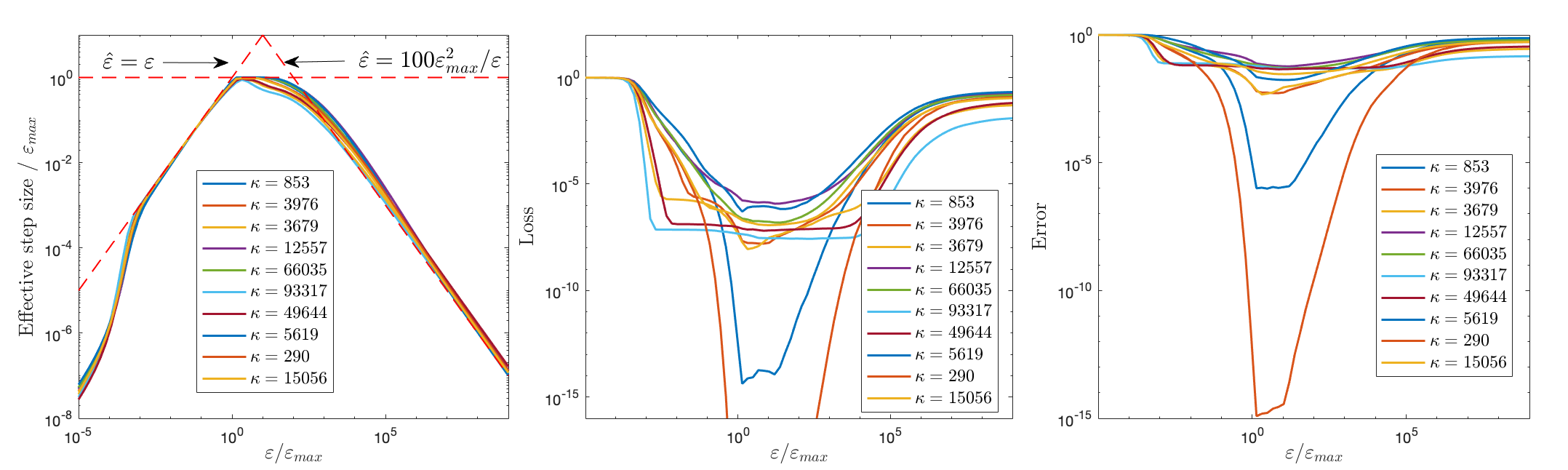}\\
    %BN1GD_random_d100_res500
    \caption{
        {%\color{red}
        Tests of BNGD on OLS model with step size $\varepsilon_a=1, a_0=0$.
        Parameters: (top) $H=$diag(linspace(1,condH,100)), (bottom) $H \in \mathbb{R}^{100\times 100}$ is a diagonal matrix with random positive entrances which has condition number $\kappa$.
        $u$ and $w_0$ is randomly chosen uniformly from the unit sphere in $\mathbb{R}^{100}$. The BNGD iterations are executed for $k=5000$ steps. The curves are averaged over the 500 independent runs.
        }
    }
    \label{fig:BN1GD_random_d100_res500}
\end{figure}

}

\end{document}